\documentclass[sigconf]{acmart}

\copyrightyear{2021} 
\acmYear{2021} 
\setcopyright{rightsretained} 
\acmConference[KDD '21]{Proceedings of the 27th ACM SIGKDD Conference on Knowledge Discovery and Data Mining}{August 14--18, 2021}{Virtual Event, Singapore}
\acmBooktitle{Proceedings of the 27th ACM SIGKDD Conference on Knowledge Discovery and Data Mining (KDD '21), August 14--18, 2021, Virtual Event, Singapore}\acmDOI{10.1145/3447548.3467437}
\acmISBN{978-1-4503-8332-5/21/08}
\usepackage[utf8]{inputenc} 
\usepackage[T1]{fontenc}
\usepackage{url}
\usepackage{soul}
\usepackage{ifthen}
\usepackage{newcommand}
\usepackage{mathtools}
\hypersetup{
	colorlinks = true,
	linkcolor = black,
	anchorcolor = black,
	citecolor = black,
	filecolor =black,
	urlcolor = black
}

\usepackage{enumerate}
\usepackage{amsmath} 
\usepackage[linesnumbered,boxed,commentsnumbered, ruled]{algorithm2e}

\usepackage{mathrsfs}
\usepackage{bm}
\usepackage{tikz}
\usepackage{algorithmic}
\usetikzlibrary{arrows}
\usepackage{subfigure}
\usepackage{graphicx,booktabs,multirow}
\usepackage{epstopdf}
\usepackage{subfigure}
\usepackage{multirow}
\usepackage{tabularx} 
\usepackage{booktabs}
\usepackage{threeparttable}

\newtheorem{lemma}{\textbf{Lemma}}
\newtheorem{theorem}{\textbf{Theore}m}

\newtheorem{assumption}{\textbf{Assumption}}

\newcommand{\norm}[2]{\left\| #1 \right\|_{#2}}
\newcommand{\normn}[2]{\| #1 \|_{#2}}

\newcommand{\E}{\mathbb{E}}

\begin{document}

\setcounter{page}{1}

\fancyhead{}
\title{Global Neighbor Sampling for Mixed CPU-GPU Training on Giant Graphs}
\date{ }
\author{Jialin Dong}
\authornote{The work was performed during an internship at AWS Shanghai AI Lab.}
\authornote{Both authors contributed equally to this research.}
\email{jialind@g.ucla.edu}
\affiliation{%
  \institution{University of California, Los Angeles}
  \country{USA}
}

\author{Da Zheng}\authornotemark[2]
\email{dzzhen@amazon.com}
\affiliation{%
  \institution{AWS AI}
  \country{USA}}

\author{Lin F. Yang}
 \email{linyang@ee.ucla.edu}
\authornote{Corresponding author.}

\affiliation{%
 \institution{University of California, Los Angeles}
  \country{USA}
}

\author{Geroge Karypis}\authornotemark[3]
 \email{gkarypis@amazon.com}
\affiliation{%
  \institution{AWS AI}
  \country{USA}}

\begin{abstract}

Graph neural networks (GNNs) are powerful tools for learning from graph data and are widely used in various applications such as social network recommendation, fraud detection, and graph search. The graphs in these applications are typically large, usually containing hundreds of millions of nodes. Training GNN models on such large graphs efficiently remains a big challenge.
Despite a number of sampling-based methods have been proposed to enable mini-batch training on large graphs, %Many sampling-based methods have been proposed to enable mini-batch training on large graphs.
%However, 
these methods have not been proved to work on truly industry-scale graphs, which require GPUs or \emph{mixed-CPU-GPU training}.
The state-of-the-art sampling-based methods are usually not optimized for these real-world hardware setups, in which data movement between CPUs and GPUs is a bottleneck.
%In practice, GPUs remain efficient hardware for training GNN models. Due to the small GPU memory
%size, state-of-the-art GNN frameworks, such as DGL and Pytorch-Geometric, keeps the graph data
%in CPU memory and performs mini-batch computation on GPUs to train large graphs with GPUs.
%We refer to this training strategy as \textit{mixed-CPU-GPU training}. The main bottleneck
%of this training strategy is data movement between CPUs and GPUs. 
To address this issue,
we propose Global Neighborhood Sampling that aims at training GNNs on giant graphs specifically for mixed-CPU-GPU training. 
The algorithm samples a global cache of nodes periodically for all mini-batches and stores them in GPUs. % beforehand. 
This global cache allows in-GPU importance sampling of mini-batches, which drastically
%we only need a small cache in a power-law graph to effectively
reduces the number of nodes in a mini-batch, especially in the input layer, to reduce data copy between CPU and GPU and mini-batch computation without compromising the training convergence rate or model accuracy.
%and avoids data movement between GPU and CPU. 
%When performing neighborhood sampling
%for mini-batches, it samples neighbors from the cache. 
%By combining it with
%Together with importance sampling,
%we show that this approach 
We provide a highly efficient implementation of this method and show that our implementation
outperforms an efficient  node-wise neighbor sampling baseline by a factor of $2\times - 4\times$
on giant graphs. It outperforms an efficient implementation of LADIES with small layers
by a factor of $2\times-14\times$ while achieving much higher accuracy than LADIES.
We also theoretically analyze the proposed algorithm and show that with cached node data of a proper size, it enjoys a comparable convergence rate as the underlying node-wise sampling method.

\end{abstract}

\maketitle

\section{Introduction}\label{sec:introduction}
Many real world data come naturally in the form of graphs; e.g., social networks, recommendation,
gene expression networks, and knowledge graphs.
%Graphs are powerful and versatile data structures to model many real world problems.
In recent years, Graph Neural Networks (GNNs) \cite{gcn, gat, graphsage} have been proposed to learn from such graph-structured data and have achieved outstanding performance. 
Yet, in many applications, graphs are usually large, containing hundreds of millions to billions of nodes and tens to hundreds of billions of edges.
Learning on such giant graphs is challenging due to the limited amount of memory available on a single GPU or a single machine. 
As such, mini-batch training is used to train GNN models on such giant graphs. 
However, due to the connectivities between nodes, computing the embeddings of a node
with multi-layer GNNs usually involves in many nodes in a mini-batch. This leads to
substantial computation and data movement between CPUs and GPUs in mini-batch training and makes training inefficient.
%\gk{You need to explain why the above statement represents a problem.}

% To handle a large graph, a line of literature (e.g., 
% \cite{roc, neugraph, tripathy2020reducing}) has developed methods in the multi-GPU setting or distributed setting  to scale GNN training on large graph data.
% These works perform full graph training on multiple GPUs or distributed memory whose
% aggregated memory fit the graph data. 
% \textcolor{blue}{Unfortunately, such methods requires large amount of data trnasfering }
% \lin{I feel it is a bit too early to mention DistDGL here. I would like to rewrite a bit by using the following outline:
% 1) DistDGL aims to solve the distributed learning prolem, but there are still issues remaining unsolved.
% 2) For non-distributed training, there are many methods using sampling. But they do not scale.
%Motivated by those, we use sampling to improve performance of DistDGL.}
%To address this issue, mini-batch training has been widely developed in training a neural
% network \cite{aligraph, agl, euler}. Recently, a  mini-batch fashion distributed GNN system, i.e.,
% DistDGL \cite{zheng2020distdgl}, improves performance via
% adopting locality-aware graph partitioning and co-locate data and communication.

To remedy this issue, various GNN training methods
have been developed to reduce the number of nodes in a mini-batch~\cite{graphsage, zeng2019graphsaint, fastgcn, zou2019layer, liu2020bandit}. Node-wise neighbor sampling used by GraphSage \cite{graphsage}
samples a fixed number of neighbors for each node independently. Even though it reduces the number
of neighbors sampled for a mini-batch, the number of nodes in each layer still grows exponentially.
FastGCN \cite{fastgcn} and LADIES \cite{zou2019layer} sample a fixed number of
nodes in each layer, which results in isolated nodes when used on large graphs. In addition, LADIES
requires significantly more computation to sample neighbors and can potentially slow down the overall
training speed. The work by Liu et al.~\cite{liu2020bandit} tries to alleviate the neighborhood explosion and reduce
the sampling variance by applying a bandit sampler. However, this method leads to very large sampling
overhead and does not scale to large graphs. These sampling methods are usually evaluated on small to medium-size graphs.
When applying them to industry-scale graphs, they have suboptimal performance or substantial computation
overhead as we discovered in our experiments.

%
%\gk{If the above statement is supported by another study, provide a reference. If not, you need to say something like "...as our experiments show..."}
%
%Furthermore, some of the sampling methods are computationally expensive leading to an overall increase in training time.
%
To address some of these problems and reduce training time, LazyGCN ~\cite{ramezani2020gcn} periodically samples a \emph{mega-batch} of nodes and reuses it to sample further mini-batches. By loading each mega-batch in GPU memory once, LazyGCN can mitigate data movement/preparation overheads. However, in order to match the accuracy of standard GNN training  LazyGCN requires very large mega-batches 
(cf. Figure~\ref{fig:lazygcn}), which makes it impractical for graphs with hundreds of millions of nodes.

We design an efficient and scalable sampling method that takes into account the characteristics
of training hardware into consideration. 
%
%\gk{Does the method only work for very large graphs?}
%
GPUs are the most efficient hardware for training GNN models. Due to the small GPU memory size, state-of-the-art GNN frameworks, such as DGL~\cite{wang2019dgl} and Pytorch-Geometric~\cite{pyg}, keep the graph data in CPU memory and perform mini-batch computations on GPUs when training GNN models on large graphs. We refer to this training strategy as \textit{mixed CPU-GPU training}. The main bottleneck of mixed CPU-GPU training is data copy between CPUs and GPUs (cf. Figure~\ref{fig:sampling}). Because mini-batch sampling occurs in CPU, we need a low-overhead sampling algorithm to enable efficient training.
Motivated by the hardware characteristics, we developed the \emph{Global Neighborhood Sampling} (GNS) approach that samples
a global set of nodes periodically for all mini-batches.
The sampled set is small so that we can store all of their node features in GPU memory
and we refer this set of nodes as \textit{cache}. % beforehand. 
%
%\gk{The term ``cache'' here is ambiguous (i.e., cache in the sense of CS systems, or cache in the sense of store) and will confuse the reader.}
%\
The cache is used for neighbor sampling in a mini-batch. Instead of sampling
any neighbors
of a node, GNS gives the priorities of sampling neighbors that exist in the cache.
This is a fast way of biasing node-wise neighbor sampling to reduce the number of
distinct nodes of each mini-batch and increase the overlap between mini-batches.
When coupled with GPU cache, this method
drastically reduces the amount of data copy between GPU and CPU to speed up training.
In addition, we deploy importance sampling that reduces the sampling variance and also allows us to use a small cache size to train models.

We develop a highly optimized implementation of GNS and compare it with efficient implementations
of other sampling methods provided by DGL, including node-wise neighbor sampling and LADIES.
We show that GNS achieves state-of-the-art model accuracy while speeding up training by a factor of $2\times - 4\times$
compared with node-wise sampling and by a factor of $2\times - 14\times$ compared with LADIES.

% Our proposed method is implemented by DistDGL.  The visual illustration of the proposed methods is presented in Fig. \ref{fig:sys} and the procedure of the proposed sampling algorithm is described as follows: 
% \begin{enumerate}
% \item Sample and cache a set of nodes and associated data in GPU, which is used for establishing the input layer of GNN mini-batches. The sampling probability is calculated according to the degree of nodes. The size of the cached nodes is dependent on the graph scale. Generally, it varies from two to ten percent of the total nodes.
% \item Sample output layers and inner layers nodes via state-of-the-art node-wise neighbor sampling algorithms and transfer associated node data to GPU.
% \item Establish mini-batch GNN on GPU. For the input layer, it picks all cached nodes that are contained in the neighbor lists of the nodes sampled in the upper layer.
% \item Each machine
% computes model gradients with respect to its own mini-batch,
% synchronizes gradients with others and updates the local model
% replica. To alleviate the variance of the gradient due to the prior cached nodes in the input layer of the sampled mini-batch, we develop an importance sampling schema in the forward pass.

% \end{enumerate}

The main contributions of the work are described below:
\begin{enumerate}
%\item We first demonstrate where the overhead lies in mixed-CPU-GPU mini-batch training.
\item We analyze the existing sampling methods and demonstrate their main drawbacks on large graphs.
\item We develop an efficient and scalable sampling algorithm that addresses the main overhead in
mixed CPU-GPU mini-batch training and show a substantial training speedup compared with efficient implementations
of other training methods.
\item We demonstrate that this sampling algorithm can train GNN models on graphs with over 111 millions of nodes and 1.6 billions of edges.
\end{enumerate}

\section{Background}\label{sec:method}
In this section, we review GNNs and several state-of-the-art sampling-based training algorithms,
including node-wise neighbor sampling methods and layer-wise importance sampling methods.
The fundamental concepts of mixed-CPU-GPU training architecture is introduced. We discuss the limitations of state-of-the-art sampling methods in mixed-CPU-GPU training.

\subsection{Existing GNN Training Algorithms}
\textbf{Full-batch GNN}
Given a graph $\mathcal{G(\mathcal{V},\mathcal{E})}$,
the input feature of node $v\in\mathcal{V}$ is denoted as $\mathbf{h}_v^{(0)}$, and the feature of the edge
between node $v$ and $u$ is represented as $\mathbf{w}_{uv}$.
The representation of node $v\in\mathcal{V}$ at layer $\ell$ can be derived from a GNN model given by:

% \vspace{-1em}

\begin{equation}\label{eq:mp-vertex}
\mathbf{h}_v^{\ell} = g(\mathbf{h}_v^{\ell-1},\bigcup_{u\in\mathcal{N}(v)} f(\mathbf{h}_u^{\ell-1}, \mathbf{h}_v^{\ell-1}, \mathbf{w}_{uv})),
\end{equation}
where $f$, $\bigcup$ and $g$ are pre-defined or parameterized functions for computing feature data, aggregating data information, and
updating node representations, respectively. For instance, in GraphSage training \cite{graphsage}, the candidate aggregator functions include mean aggregator \cite{gcn}, LSTM aggregator \cite{hochreiter1997long}, and max pooling aggregator \cite{qi2017pointnet}. The function $g$ is set as nonlinear activation function.

Given training dataset $\{(\xb_i,y_i)\}_{v_i\in\cV_s}$, the parameterized functions will be learned by minimizing the loss function:
 \begin{equation}
\cL = \frac{1}{|\cV_s|}\sum_{v_i\in \cV_s} \ell(y_i,\zb_i^{L}),
 \end{equation}
where $\ell(\cdot,\cdot)$ is a loss function, $\zb_i^{L}$ is the output of GNN with respect to the node $v_i\in\cV_s$ where $\cV_S$ represents the set of training nodes. For full-batch optimization, the loss function is optimized by
gradient descent algorithm where the gradient with respect to each node $v_i\in \cV_S$ is computed as $\frac{1}{|\cV_s|}\sum_{v_i\in \cV_S} \nabla\ell(y_i,\zb_i^{(L)})$. During the training process, full-batch GNN requires to store and aggregate representations of all nodes across all layers. 
The expensive computation time and memory costs prohibit full-batch GNN from handling large graphs. 
Additionally, the convergence rate of full-batch GNN is slow because model parameters are updated only once at each epoch.

% \gk{Is the paragraph that follows part of full-batch GNN training or part of mini-batch training?}
% %
\noindent
\textbf{Mini-batch GNN}
To address this issue, a mini-batch training scheme has been developed which optimizes via
mini-batch stochastic gradient descent $\frac{1}{|\cV_B|}\sum_{v_i\in \cV_B} \nabla \ell(y_i,\zb_i^L)$
where $\cV_B\in\cV_S$. These methods first uniformly sample a set of nodes from the training set, known as
\textit{target nodes}, and sample neighbors of these target nodes to form a mini-batch.
The focus of the mini-batch training methods is to reduce the number of
neighbor nodes for aggregation via various sampling strategies to reduce the memory and
computational cost. The state-of-the-art sampling algorithm is discussed in the sequel.

\noindent
\textbf{Node-wise Neighbor Sampling Algorithms.}
Hamilton et al.~\cite{graphsage} proposed an unbiased sampling method to reduce the number of
 neighbors for aggregation via neighbor sampling.  It randomly selects
at most $s_{\text{node}}$ (defined as \textit{fan-out} parameter) neighborhood nodes for every target node;
followed by computing the representations of target nodes via
aggregating feature data from the sampled neighborhood nodes. Based on the notations in (\ref{eq:mp-vertex}), the representation of node $v\in\mathcal{V}$ at layer $\ell$ can be described as follows: 
\begin{equation}
\mathbf{h}_v^{\ell} = g\bigg(\mathbf{h}_v^{\ell-1},\bigcup_{u\in\mathcal{N}_{\ell}(v)} f\big(\frac{1}{s_{\text{node}}}\mathbf{h}_u^{\ell-1}, \mathbf{h}_v^{\ell-1}\big)\bigg),
\end{equation}
where $\mathcal{N}_{\ell}(v)$ is the sampled neighborhood nodes set at $\ell$-th layer such that $|\mathcal{N}_{\ell}(v)| =s_{\text{node}} $.
The neighbor sampling procedure is repeated recursively on target nodes and their sampled neighbors when dealing with multiple-layer GNN. 
Even though node-wise neighbor sampling scheme addresses the memory issue of GNN, there exists excessive computation under this scheme because the scheme still results in exponential growth of neighbor nodes with the number of layers. This yields a large volume of data movement between CPU and GPU for mixed CPU-GPU training.

% \gk{You need to have a section that discusses variance reduction. It should come before you start mentioning it (e.g., you discuss it in layer-wise sampling.}\jl{I added more explanations on Layer-wise Importance Sampling Algorithms}

\noindent
\textbf{Layer-wise Importance Sampling Algorithms.}
 To address the scalability issue, Chen et al.~\cite{fastgcn} proposed an advanced layer-wise method called FastGCN. Compared with node-wise sampling method, it yields extra variance when sampling a fixed number of nodes for each layer. To address the variance issue, it performs degree-based importance sampling on each layer. The representation of node $v\in\mathcal{V}$ at layer $\ell$ of FastGCN model is described as follows:
\begin{equation}
\mathbf{h}_v^{\ell} = g\bigg(\mathbf{h}_v^{\ell-1},\bigcup_{u\in q(v)} f\big(\frac{1}{s_{\text{layer}}}\mathbf{h}_u^{\ell-1}/{ q_u}, \mathbf{h}_v^{\ell-1}\big)\bigg),
\end{equation}
where the sample size denotes as $s_{\text{layer}}$, $q(v)$ is the distribution over $v\in\mathcal{V}$ and $q_u$
is the probability assigned to node $u$.
A major limitation is that FastGCN performs sampling on every layer independently, which yields approximate embeddings with large variance. Moreover, the subgraph sampled by FastGCN is not representative to the original graph. This leads to
poor performance and the number of sampled nodes required to guarantee convergence during training process is large.

% \gk{Do not just the '[5] does x' or 'the work by [7] did that', the [5] and [7] are citations they are not people. If you need to refer to them, use the last name of the authors, just like what I did with LADIES bellow.}

The work by Zhou et al.~\cite{zou2019layer} proposed a sampling algorithm known as LAyer-Dependent
ImportancE Sampling (LADIES) to address the limitation of FastGCN and exploit the connection
between different layers.
Specifically, at $\ell$-th layer, LADIES samples nodes reachable from the nodes in the previous layer. However, this method \cite{zou2019layer} comes with cost.
To ensure the node connectivity between layers, the method needs to
extract and merge the entire neighborhood of all nodes in the previous layer and compute the sampling probability
for all candidate nodes in the next layer. Thus, this sampling method has significantly higher computation
overhead. Furthermore, when applying this method on a large graph, it still constructs a mini-batch with
many isolated nodes, especially for nodes in the first layer (Table \ref{tab:ladies}).

\noindent
\textbf{LazyGCN.}
Even though layer-wise sampling methods effectively address the neighborhood explosion issue, they failed to investigate computational overheads in preprocessing data and loading fresh samples during training. Ramezan et al. \cite{ramezani2020gcn} proposed a framework called LazyGCN which decouples the frequency of sampling from the sampling strategy. It periodically samples \textit{mega-batches} and effectively reuses \textit{mega-batches} to generate mini-batches and alleviate the preprocessing overhead.
There are some limitations in the LazyGCN setting. First, this method requires large mini-batches to guarantee model accuracy. For example, their experiments on Yelp and Amazon dataset use the batch size
of 65,536. This batch size is close to the entire training set, yielding overwhelming overhead
in a single mini-batch computation. Its performance
deteriorates for smaller batch sizes even with sufficient epochs (Figure \ref{fig:lazygcn}). Second, their evaluation is based on
inefficient implementations with very large sampling overhead. In practice, the sampling computation overhead
is relatively low in the entire mini-batch computation (Figure \ref{fig:sampling}) when using
proper development tools.

Even though both GNS and LazyGCN cache data in GPU to accelerate computation
in mixed CPU-GPU training, they use very different strategies for caching. LazyGCN caches
the entire graph structure of multiple mini-batches sampled by node-wise neighbor sampling or layer-wise
sampling and suffers from the problems in these two algorithms. Due to the exponential growth
of the neighborhood size in node-wise neighbor sampling, LazyGCN cannot store a very large mega-batch
in GPU and can generate few mini-batches from the mega-batch. Our experiments show that LazyGCN
runs out of GPU memory even with a small mega-batch size and mini-batch size on large graphs
(OAG-paper and OGBN-papers100M in Table \ref{tab: exp-dataset}). Layer-wise sampling may result in many 
isolated nodes in a mini-batch. In addition, LazyGCN uses the same sampled graph structure when 
generating mini-batches from mega-batches, this potentially leads to overfit. In contrast, GNS cache 
nodes and use the cache to reduce the number of nodes in a mini-batch; GNS always sample a 
different graph structure for each mini-batch and thus it is less likely to overfit.

\subsection{Mixed CPU-GPU training}
Due to limited GPU memory, state-of-the art GNN framework (e.g., DGL~\cite{wang2019dgl}  and Pytorch Geometric~\cite{pyg}) train GNN models on large graph data by storing the whole graph data in CPU memory and performing mini-batch computation on GPUs. This allows users to take advantage of large CPU memory and use GPUs to accelerate GNN training. In addition, mixed CPU-GPU training makes it easy to scale GNN training to multiple GPUs or multiple machines~\cite{zheng2020distdgl}. 
%
%GK: This already defined in the introduction.
%We refer to this training strategy as \textit{mixed-CPU-GPU training}.

Mixed CPU-GPU training strategy usually involves six steps: 
\textit{1)} sample a mini-batch from the full graph, 
\textit{2)} slice the node and edge data involved in the mini-batch from the full graph,
\textit{3)} copy the above sliced data to GPU, 
\textit{4)} perform forward computation on the mini-batch, 
\textit{5)} perform backward propagation, and
\textit{6)} run the optimizer and update model parameters.
Steps~1--2 are done by the CPU, whereas steps~4--6 are done by the GPU. 

We benchmark the mini-batch training of GraphSage~\cite{graphsage} models with node-wise neighbor sampling provided by DGL, which provides very efficient neighbor sampling implementation and graph kernel computation for GraphSage. 
Figure~\ref{fig:sampling} shows the breakdown of the time required to train GraphSage on the OGBN-products graph and the OAG-paper graph (see Table~\ref{tab: exp-dataset} for dataset information). Even though sampling happens in CPU, its computation accounts for 10\% or less with sufficient optimization and parallelization. However, the speed of copying node data in CPU (step~2) is limited by the CPU memory bandwidth and moving data to GPU (step~3) is limited by the PCIe bandwidth. Data copying accounts for most of the time required by mini-batch training. 
For example, the training spends ~60\% and ~80\% of the per mini-batch time in copying data from CPU to GPU on OGBN-products and OAG-paper, respectively.
The training on OAG-paper takes significantly more time on data copy because OAG-paper has 768-dimensional BERT embeddings~\cite{bert} as node features, whereas OGBN-products uses 100-dimensional node features.
%
%\gk{In order for the statement that follows to be full-proof, you need to address the following: (i) Report the fraction of the PCIe B/W that you utilize during the CPU-GPU copying. If this is close to peak, then there is no room for further optimizing how data is transferred. (ii) You cannot overlap this CPU-GPU copying phase with computations, as this will lead to operating on stale model parameters.}
% TODO: we cannot address this comment now. We need to evaluate this sampling in DistDGL
% to address the comment above.

\begin{figure}
\centering
\includegraphics[width=0.65\linewidth]{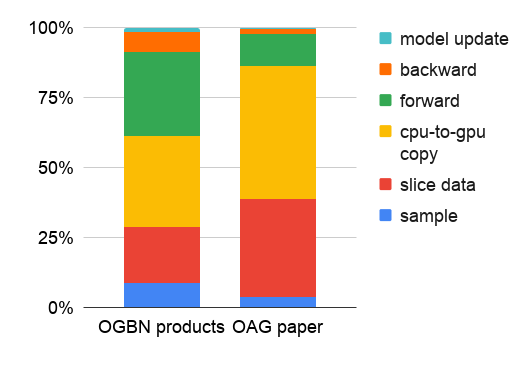}
\caption{Runtime breakdown (\%) of each component in mini-batch training for
an efficient GraphSage implementation in DGL.}
\label{fig:sampling}
\end{figure}

%	sample	slice data	cpu-to-gpu copy	forward	backward	step
%graphsage on OGB product graph	0.015	0.031	0.0435	0.0285	0.0085	0.002
%rgcn on OAG graph	0.0087	0.0328	0.1044	0.0186	0.0108	0.002

These results show that when the different components of mini-batch training are highly-optimized, the main bottleneck of mixed CPU-GPU training is data copy (both data copy in CPU and between CPU and GPUs). To speed up training, it is essential to reduce the overhead of data copy, without significantly increasing the overhead of other steps.

\section{Global Neighbor Sampling (GNS)}
To overcome the drawbacks of the existing sampling algorithms and tackle the unique problems in mixed CPU-GPU training, we developed a new sampling approach, called \emph{Global Neighborhood Sampling} (GNS), that has low computational overhead and reduces the number of nodes in a mini-batch without compromising the model accuracy and convergence rate. Like node-wise and layer-wise sampling, GNS uses mini-batch training to approximate the full-batch GNN training on giant graphs. 
\begin{table*}[t]
	\caption{Summary of Notations and Definitions} 
	\label{tab:notation}
	\centering
	\small
	\begin{tabular}{l|p{10cm}}
		\toprule
		$\cG = (\cV,\mathcal{E})$ & $\cG$ denotes the graph consist of set of $|\mathcal{V}|$ nodes and $| \mathcal{E}$ | edges.\\ \hline
		$L,K$ & $L$ is the total number of layers in GCN, and  $K$ is the dimension of embedding vectors (for simplicity, assume it is the same across all layers).\\ \hline
		$b,s_{node},s_{layer}$ & For batch-wise sampling, $b$ denotes the batch size, $s_{node}$ is the number of sampled neighbors per node for node-wise sampling, and $s_{layer}$ is number of sampled nodes per layer for layer-wise sampling. \\ \hline
		$\mathcal{N}(v)$ & Denotes the set of neighbors of node $v\in\mathcal{V}$.\\ \hline
		$\mathcal{N}_{\ell}(v)$ & Denotes the set of sampled neighbors of node $v\in\mathcal{V}$ at $\ell$-th layer.\\ \hline
		$\mathcal{N}_{\ell}^{\text{u}}(v)$ & Denotes the set of neighbors of node $v\in\mathcal{V}$ at $\ell$-th layer sampled uniformly at random.\\ \hline
		$\mathcal{N}_{C}(v)$ & Denotes the set of neighbors of node $v\in\mathcal{V}$ in the cache. \\ \hline
		$\mathcal{C},p^{\text{cache}}_{v}$ & Denotes the set of  cached nodes which are sampled from $\mathcal{V}$ corresponding to the probability of  $p^{\text{cache}}_{v}$ for $v\in\mathcal{V}$.\\ \hline
		 $p^{(\ell)}_v $& Denotes importance sampling coefficients with respect to the node $v\in\mathcal{V}$ at $\ell$-th layer in Algorithm \ref{alg:full}.\\ \hline
		 $\cV_S$, $|\cV_S|$ & Denotes the training set and the size of the training set. \\ \hline
		 $\text{target node}$& The node in the mini-batch where the mini-batch is sampled at random from the training node set. \\ \hline
		 $\cV_B$, $|\cV_B|$ & Denotes the set of target nodes and the number of target nodes in a mini-batch.
		\\ \bottomrule
	\end{tabular}
\end{table*}

\subsection{Overview of GNS}

Instead of sampling neighbors independently like node-wise neighbor sampling, GNS periodically samples
a global set of nodes following a probability distribution $\mathcal{P}$ to assist in neighbor sampling.
$\mathcal{P}_i$ defines the probability of node $i$ in the graph being sampled and placed in the set.
Because GNS only samples a small number of nodes to form the set, we can copy all of the node features
in the set to GPUs. Thus, we refer to the set of nodes as \textit{node cache} $\mathcal{C}$.
When sampling neighbors of a node, GNS prioritizes the sampled neighbors from the cache and
samples additional neighbors outside the cache only if the cache does not provide sufficient
neighbors.

% \gk{The paragraph that follows is out of place here. I think you put it here because you want to refer to the algorithm, but you should refer to the algorithm only when you have finished describing all the components.}

Because the nodes in the cache are sampled, we can compute the node sampling probability from
the probability of a node appearing in the cache, i.e., $\mathcal{P}$. 
%
%\gk{The notation is ambiguous. What is dimensionality of $\mathcal{P}$ and what does it encode?}
%
We rescale neighbor embeddings by importance sampling coefficients $p_{u}^{(\ell)}$ from  $\mathcal{P}$ in the mini-batch forward propagation
so that the expectation of the aggregation of sampled neighbors is the same as the aggregation of
the full neighborhood.
\begin{align}
  \mathbb{E}\left(\sum_{u \in \mathcal{N}_{\ell}(v)} p_u^{(\ell)} * h_u^{\ell}\right) = \sum_{u \in \mathcal{N}(v)} h_u^\ell
\end{align}

\noindent Algorithm~\ref{alg:full} illustrates the entire training process.
%
%\gk{Have you previously described what ``adjust'' mean and what problem you are solving by doing that?}
%

In the remaining sections, we first discuss the cache sampling in Section~\ref{sec:cache}.
We explain the sampling procedure in Section~\ref{sec:sampling}.
To reduce the variance in GNS, an importance sampling scheme is further developed in Section~\ref{sec: IS}. We then establish the convergence rate of GNS which is inspired by the paper \cite{ramezani2020gcn}. It shows that under mild assumption, GNS enjoys comparable convergence rate as underlying node-wise sampling in training, which is demonstrated in Section~\ref{analysis}. The notations and definition used in the following is summarized in Table \ref{tab:notation}.

\begin{algorithm}[tb] 
	\caption{Minibatch Training with Global Neighbor sampling for GNN on Node Classification}
	\label{alg:full}
	\SetKwInOut{Input}{Input}\SetKwInOut{Output}{Output}
	\Input{Graph $\mathcal{G(\mathcal{V},\mathcal{E})}$;\\
		list of target nodes of minibatches $\{ \mathcal{B}_1,\cdots, \mathcal{B}_M  \}$;\\
		input features $\{\mathbf{x}_v, \forall v\in \mathcal{V}\}$; \\
		 number of epochs $T$;\\
		depth $L$; weight matrices $\mathbf{W}^{\ell}, \forall \ell \in \{1,...,L\}$; \\
		cache sampling probability $\mathcal{P}$; \\
		nonlinear activation function $g$; \\
		differentiable aggregator functions $f_{\ell}, \forall \ell \in \{1,...,L\}$.\\}

	\Output{Vector representations $\mathbf{z}_v$ for all $v \in \mathcal{B}$}
	\begin{algorithmic}[1] 
		\FOR {$t=0$ to $T$}
		\STATE $\mathcal{C} \leftarrow sample\_cache(\mathcal{V}, \mathcal{P}, \{ \mathcal{B}_1,\cdots, \mathcal{B}_M  \})$
		\FOR {$\mathcal{B}\in \{ \mathcal{B}_1,\cdots, \mathcal{B}_M  \}$  }
		\STATE $\mathcal{B}^L \leftarrow \mathcal{B}$
		\FOR {$\ell=L...1$}\label{line:for}
		\STATE $B^{\ell-1} \leftarrow \{\}$ 
		\FOR{$u \in \mathcal{B}^{\ell}$}\label{line:re}
		\STATE $\mathcal{N}_{\ell}(u), \mathcal{P}_{\ell}(u) \leftarrow sample(\mathcal{N}(u), C)$
		\STATE $\mathcal{B}^{\ell-1} \leftarrow \mathcal{B}^{\ell-1} \cup \mathcal{N}_{\ell}(u)$;
		\STATE $\mathcal{P}^{\ell-1} \leftarrow \mathcal{P}^{\ell-1} \cup \mathcal{P}_{\ell}(u)$
		\ENDFOR\label{line:rend}
		\ENDFOR\label{line:endfor}
		
		\STATE $\mathbf{h}^0_u \leftarrow \mathbf{x}_v, \forall v \in \mathcal{B}^0$
		\FOR{$\ell =1...L$}\label{line:for1}
		\FOR{$u \in \mathcal{B}^{\ell}$}
		\STATE Compute importance sampling coefficients $p^{(\ell-1)}_{u'}$ for $\forall u' \in \mathcal{N}_{\ell}(u)\}$ 
		\STATE $\mathbf{h}^{\ell}_{\mathcal{N}(u)} \leftarrow f_{\ell}(\{p^{(\ell-1)}_{u'}\mathbf{h}_{u'}^{\ell-1},
\forall u' \in \mathcal{N}_{\ell}(u)\})$
		\STATE $\mathbf{h}^{\ell}_u \leftarrow g\left(\mathbf{W}^{\ell}\cdot (\mathbf{h}_u^{\ell-1}, \mathbf{h}^{\ell}_{\mathcal{N}(u)})\right)$
		\ENDFOR
		\ENDFOR\label{line:endfor1}
		\ENDFOR
		\ENDFOR
	\end{algorithmic} 
\end{algorithm}

\subsection{Sample Cache}\label{sec:cache}
%Previous sampling-based method is implemented by
% a uniform sampling procedure to establish neighborhood sampling set. Where the sample size is greater than the node's degree, it samples with replacement from the neighborhood nodes. As the depth of layers increases, the number of input layer nodes increases, which brings extra computation cost in training. It can also increases the volume of data movement between CPU and GPU when executing mixed-CPU-GPU training. Moreover, the random neighborhood sampling produces new set at every iteration, which requires to frequently update the data information stored in GPU. This induces the high frequency of data movement between CPU and GPU. To alleviate both the volume and frequency of data movement between CPU and GPU for mixed-CPU-GPU training, the Global Neighborhood Sampling (GNS) algorithm is developed reduce the input nodes in the minibatch set by developing a novel sampling procedure which is described as follows.
 
%  \gk{Set $C$ is the key to your algorithm. This is buried in the description of the approach. You should present your method by first giving the high-level overview of the idea followed by the details. For example: To reduce data transfer costs, it breaks the epoch into different sub-epochs. In each sub-epoch, it first uses a biased sampling approach to select a set of nodes $C$, whose embeddings are loaded and cached in the GPU, and then generates mini-batches by restricting the neighboring nodes to belong to $C$....}

GNS periodically constructs a cache of nodes $\mathcal{C}$ to facilitate neighbor sampling in mini-batch construction. GNS uses a biased sampling approach to select a set of nodes $C$ that, with high probability, can be reached from nodes in the training set. 
% %
% \gk{You need to have a ``Definitions \& Notations'' section. Among other things, you should define the term \textbf{training} node (or \textbf{target} node), and use it when you describe how GNS works.}
% %
The features of the nodes in the cache are loaded into GPUs beforehand.

Ideally, the cache needs to meet two requirements: 1) in order to keep the entire cache in the GPU memory,
the cache has to be sufficiently small; 2) in order to have sampled neighbors to come from the cache,
the nodes in the cache have to be reachable from the nodes in the training set with a high probability.
%
%\gk{Before you start describing what you did, you need to discuss what is/are the problems of using a naive method for selecting the node cache. You need to specify the properties that good node caches need to have.}
%

Potentially, we can uniformly sample nodes to form the cache, which may require a large number of nodes
to meet requirement 2.
Therefore, we deploy two approaches to define the sampling probability for the cache.
If majority of the nodes in a graph are in the training set, we define the sampling
probability based on node degree. For node $i$, the probability of being sampled in the cache is given by
\begin{align}\label{sp}
  p_{i} =   \mathrm{deg}(i)/\sum_{k\in\mathcal{V} } \mathrm{deg}(k).
\end{align}

\noindent For a power-law graph, we only need to maintain a small cache of nodes to cover majority of the nodes in the graph.

If the training set only accounts for a small portion of the nodes
in the graph, we use short random walks to compute the sampling probability. Define $\mathcal{N}_{\ell}(v)$ as the number of sampled neighbor nodes corresponding to node $v\in\mathcal{V}$ in each layer,
\begin{align}
\mathbf{d} = [\mathcal{N}_{\ell}(v_1)/\mathrm{deg}(v_1),\cdots,\mathcal{N}_{\ell}(v_{|\mathcal{V|}})/\mathrm{deg}(v_{|\mathcal{V|}})]^{\top},
\end{align}
The node sampling probability $\bm{P}^{\ell}\in\mathbb{R}^{|\mathcal{V}|}$ for the $\ell$-th layer is represented as
\begin{align}
\bm{P}^{\ell} = (\mathbf{D}\mathbf{A } + \mathbf{I})\bm{P}^{\ell-1},
\end{align}
where $\mathbf{A}$ is the adjacency matrix and
$\mathbf{D} = \mathrm{diag}(\mathbf{d})$.
$\bm{P}^0$ is 
\begin{align}
p_i^0= \left\{
\begin{array}{ll}
      \frac{1}{|\cV_S|}, & \text{ if } i \in \cV_S \\
      0,                       & \text{ otherwise. }\\
\end{array}
\right. 
\end{align}

The sampling probability for the cache is set as $\bm{P}^L$, where $L$ is the number of layers in the multi-layer GNN model.
% \dz{we need to formally define $L$ because $L$ isn't the original random walk matrix when we deploy neighbor sampling}

% \dz{@George, What is the formula to compute the cache size?}

As the experiments will later show (cf. Section~\ref{sec:simulation}), the size of the cache $\mathcal{C}$ can be as small as $1\%$ of the number of nodes ($|\mathcal{V}|$) without compromising the model accuracy and convergence rate.

\subsection{Sample Neighbors with Cache}\label{sec:sampling}
When sampling $k$ neighbors for a node, GNS first restricts sampled neighbor nodes from the cache $\mathcal{C}$.
If the number of neighbors sampled from the cache is less than $k$, it samples remaining neighbors uniformly at random
from its own neighborhood.
%
%\gk{Can you provide additional detail for the above. If a node need to sample $k$ neighbors, and $C$ has less than $k$ neighbors, does it back-fill the rest by sampling from the rest of its neighbors?}
%
%The sampling probability of a neighbor $j$ in the neighborhood of node $i$ is defined as below:
%\begin{align}\label{pr_j}
%\Pr\{j\}= \left\{
%\begin{array}{ll}
%      p_j, & \text{ if } j \in \mathcal{C} \text{ and } k \geq \mathcal{N}_C(i) \\
%      p_j \times \frac{k}{\mathcal{N}_C(i)}, & \text{ if } j \in \mathcal{C} \text{ and } k < \mathcal{N}_C(i) \\
%      \frac{k}{\mathcal{N}(i)}, & \text{ if } j \not\in \mathcal{C}. \\
%\end{array} 
%\right. 
%\end{align}

A simple way of sampling neighbors from the cache is to compute the overlap of the neighbor list
of a node with the nodes in the cache. Assuming one lookup in the cache has $O(1)$ complexity,
this algorithm will result in $O(|\mathcal{E}|)$ complexity, where $|\mathcal{E}|$ is the number
of edges in the graph. However, this complexity is significantly larger than the original node-wise
neighbor sampling $O(\sum_{i \in \mathcal{V}_B} \min(k, |\mathcal{N}(i)|))$ in a power-law graph,
where $\mathcal{V}_B$ is the set of target nodes in a mini-batch and $|\mathcal{N}(i)|))$ is
the number of neighbors of node $i$.
Instead, we construct an induced subgraph $\mathcal{S}$ that contains the nodes in the cache and their neighbor nodes. This is done once, right after we sample nodes in the cache.  For an undirected graph, this subgraph contains the neighbors of all nodes that reside in the cache.  During neighbor sampling, we can get the cached neighbors of node $i$ by reading the neighborhood  $\mathcal{N}_S(i)$ of node $i$ in the subgraph. Constructing the subgraph $\mathcal{S}$ is much more lightweight, usually $\ll O(|\mathcal{E}|)$.

We parallelize the sampling computations with multiprocessing. That is, we create a set of processes
to sample mini-batches independently and send them back to the trainer process for mini-batch
computation. The construction of subgraphs $\mathcal{S}$ for multiple caches
$\mathcal{C}$ can be parallelized.

\subsection{Importance Sampling Coefficient} 
\label{sec: IS}

When nodes are sampled from the cache, it yields substantial variance compare to the uniform sampling method~\cite{graphsage}.
% %
% \gk{The reason that this sampling can increase variance needs to be explained either here or in the background section. If it is in the background section, provide a backward reference to that discussion here to remind the reader.}
% \gk{I do not understand the ``$(l+1)$th nodes' neighborhood nodes''.}
% %
To address this issue, we aim to assign importance weights to the nodes, thereby rescaling the neighbor features to approximate the expectation of the uniform sampling method. 
% %
% \gk{You need to describe why assigning importance weights to the nodes does alleviate the variance problem.}
% %
To achieve this, we develop an importance sampling scheme that aggregates the neighboring feature vectors with corresponding importance sampling coefficient, i.e.,
\begin{align}
\mathbf{h}^{\ell}_{\mathcal{N}(u)} \leftarrow f_{\ell}(\{1/p^{(\ell-1)}_{u'}\cdot\mathbf{h}_{u'}^{\ell-1},
\forall u' \in \mathcal{N}_{\ell}(u)\}).
\end{align}
To establish the importance sampling coefficient, we begin with computing the probability of the sampled node $u'\in \mathcal{N}_{\ell}(u)$ being contained in the cache, given by
\begin{align}
p^{\mathcal{C}}_{u'} = 1-(1-p_{u'})^{|\mathcal{C}|},
\end{align}
where the sampling probability $p_{u'}$ refers to (\ref{sp}) and $|\mathcal{C}|$ denotes the size of cache set. Recall that the number of sampled nodes at $\ell$-layer, i.e., $k$, and $\mathcal{N}_C(i)$ in (\ref{pr_j}), the importance sampling coefficient $p^{(\ell-1)}_{u'}$ can be represented as
\begin{align}
p^{(\ell-1)}_{u'} = p^{\mathcal{C}}_{u'}\frac{k}{\min\{k,\mathcal{N}_C(i)\}}.
\end{align}

\subsection{Theoretical Analysis}\label{analysis}
% \dz{let's have some analysis in variance as well. ideally, we should have smaller variance than simple neighbor sampling. But it should be true.}
In this section, we establish the convergence rate of GNS which is inspired by the work of Ramezani et al.~\cite{ramezani2020gcn}. It shows that under mild assumption, GNS enjoys comparable convergence rate as underlying node-wise sampling in training. Here, the convergence rate of GNS mainly depends on the graph degree and the size of cached set.
We focus on a two-layer GCN for simplicity and denote the loss functions of full-batch, mini-batch and proposed GNS as

\begin{align}
J(\boldsymbol{\theta}) &=\frac{1}{N} \sum_{i \in \mathcal{V}} f_{i}\left(\frac{1}{|\mathcal{N}(i)|} \sum_{j \in \mathcal{N}(i)} \frac{1}{|\mathcal{N}(j)|}\sum_{k \in \mathcal{N}(j)}g_{jk}(\boldsymbol{\theta})\right) \\
J_{\mathcal{B}}(\boldsymbol{\theta}) &=\frac{1}{B} \sum_{i \in \mathcal{V}_{\mathcal{B}}} f_{i}\left(\frac{1}{|\mathcal{N}(i)|} \sum_{j \in \mathcal{N}(i)} \frac{1}{|\mathcal{N}(j)|}\sum_{k \in \mathcal{N}(j)}g_{jk}(\boldsymbol{\theta})\right) \\
&\widetilde{J}_{\mathcal{B}}(\boldsymbol{\theta})=\notag \\\frac{1}{B} \sum_{i \in \mathcal{V}_{\mathcal{B}}}&f_{i}\bigg(\frac{1}{|\mathcal{N}^{\text{u}}_2(i)\cap \mathcal{C}|} \sum_{j \in \mathcal{N}^{\text{u}}_2(i)\cap\mathcal{C}} \frac{1}{|\mathcal{N}^{\text{u}}_1(j)\cap\mathcal{C}|}\notag\\&\sum_{k \in \mathcal{N}^{\text{u}}_1(j)\cap\mathcal{C}}p_k^{(1)}g_{jk}( \boldsymbol{\theta})\bigg),
\end{align}
respectively, where the outer and inner layer function are defined as $f(\cdot) \in \mathbb{R}$ and $g(\cdot) \in \mathbb{R}^{n},$ and their gradients as $\nabla f(\cdot) \in \mathbb{R}^{n}$ and $\nabla g(\cdot) \in \mathbb{R}^{n \times n}$, respectively. Specifically, the function $g_{jk}(\cdot)$ depends on the nodes contained in two layers. For simplicity, we denote $\mathcal{N}
^{\mathcal{C}}(j):=\mathcal{N}^{\text{u}}_1(j)\cap\mathcal{C}$. We denote $|\mathcal{N}(i)| = N^i, |\mathcal{N}^{\text{u}}_{\ell}(i)| = N^i_{\ell},|\mathcal{N}
^{\mathcal{C}}(j)| =N_{\mathcal{C}}^{j}  $ in the following.

The following assumption gives the Lipschitz continuous constant of the gradient of composite function $J(\theta)$, which plays a vital role in theoretical analysis.

\begin{assumption}\label{ass:lp}
	Suppose $f(\cdot)$ is $L_{f}$-Lipschitz continuous, $g(\cdot)$ is $L_{g}$-Lipschitz continuous, $\nabla f(\cdot)$ is $L'_{f}$-Lipschitz continuous, $\nabla g(\cdot)$ is $L'_{g}$-Lipschitz continuous.
\end{assumption}

\begin{theorem}\label{theorem1}
	Denote  $N^i_{\ell}$ as the number of the neighborhood nodes with respect to $i\in\mathcal{V}$ sampling uniformly at random at $\ell$-th layer.  The cached nodes in the set $\mathcal{C}$ with the size of $ |\mathcal{C}|$ are sampled without replacement according to $p^{\text{cache}}_{v}$. The dimension of node feature is denoted as $n$ and the size of  mini-batch is denoted as $B$. Define  $\widetilde{C} = |\mathcal{C}|/|\mathcal{V}|$ and  $C_d = \sum_{v_i\in\mathcal{V} } \mathrm{deg}(v_i)/|\mathcal{V}|$ with the constant $c>0$. Under Assumption \ref{ass:lp}, with probability exceeding $1-\delta$, GNS optimized by stochastic gradient descent can achieve
	\begin{align}
	\mathbb{E}\left[\|\nabla J(\hat{\boldsymbol{\theta}})\|^{2}\right] \leq \mathcal{O}\left(\sqrt{\frac{\mathrm{MSE}}{t}}\right),
	\end{align}
	where $\hat{\boldsymbol{\theta}}=\min _{t} \mathbb{E}\left[\left\|\nabla J\left(\boldsymbol{\theta}_{t}\right)\right\|\right]$ with $\boldsymbol{\theta}_t = \{\mathbf{W}^{\ell}_t\}_{\ell=1}^L$ and 
	\begin{align}
	\mathrm{MSE} 
	\leq&\mathcal{O}\left( L^{'2}_{f}  \frac{\log (4 n / \delta)+1 / 2}{B}\right)+\mathcal{O}\left( L_f^ {'2} L_g^ {4}  \frac{\log (4 n / \delta)+1 / 2}{c\widetilde{C} C_dN_1^jN_2^i} \right)\notag\\
	&+\mathcal{O}\left(   L_g^ {'2}  L_f^ {2}   \frac{\log (4 n / \delta) }{c\widetilde{C} C_dN_1^jN_2^i} \right).
	\end{align}
\end{theorem}
\begin{proof}
The details on the proof of Theorem \ref{theorem1} is provided in Appendix \ref{app_proof}.
\end{proof}

\textbf{Variance of GNS}
We aim to derive the average variance of the embedding for the output nodes at each layer.
Before moving forward, we provide several useful definitions. Let $B$ denote the size of  the nodes in one layer. Consider the underlying embedding:
$
\mathbf{Z}=\mathbf{L H \Theta},
$
where $\mathbf{L}$ is the Laplacian matrix, $\mathbf{H}$ denotes the feature matrix and $\mathbf{\Theta}$ is the weight matrix, Let $\tilde{\mathbf{Z}}\in\mathbb{R}^{B\times d}$ with the dimension of feature $d$ denote the estimated embedding derived from the sample-based method. Denote $\mathbf{P}\in\mathbb{R}^{B\times |\mathcal{V}|}$ as the row selection matrix which samples the embedding from the whole embedding matrix. The variance can be represented as $\mathbb{E}[\|\tilde{\mathbf{Z}}-\mathbf{P} \mathbf{Z}\|_{F}]$. Denote $\mathbf{L}_{i, *}$ as the i-th row of matrix $\mathbf{L}, \mathbf{L}_{*, j}$ is the $\mathrm{j}$ -th column of matrix $\mathbf{L},$ and $\mathbf{L}_{i, j}$ is the element at the position $(i, j)$ of matrix $\mathbf{L}$.

For each node at each layer, its embedding is estimated based on its neighborhood nodes established from the cached set $\mathcal{C}$. Based on the Assumption 1 in \cite{zou2019layer}, we have
 \begin{align}
 	&\mathbb{E}\left[\|\tilde{\mathbf{Z}}-\mathbf{P Z}\|_{F}^{2}\right] \notag\\
 	=&  \sum_{i=1}^{|\mathcal{V}|}q_i\cdot  \mathbb{E}\left[\left\|\tilde{\mathbf{Z}}_{i, *}-\mathbf{Z}_{i, *}\right\|_{2}^{2}\right] \notag\\
 	=& \sum_{i=1}^{|\mathcal{V}|} {q_i\left\|\mathbf{L}_{i, *}\right\|_{0}}\left(\sum_{j=1}^{|\mathcal{V}|}p_{ij}\cdot s_j\left\|\mathbf{L}_{i, j} \mathbf{H}_{j, *} \mathbf{\Theta}\right\|_{2}^{2}-\left\|\mathbf{L}_{i, *} \mathbf{H} \mathbf{\Theta}\right\|_{F}^{2}\right) \notag\\
 	=& \sum_{i=1}^{|\mathcal{V}|}q_i\left\|\mathbf{L}_{i, *}\right\|_{0}\big(\sum_{j=1}^{|\mathcal{V}|}p_{ij}\cdot s_j\left\|\mathbf{L}_{i, j} \mathbf{H}_{j, *} \mathbf{\Theta}\right\|_{2}^{2}-\notag\\
 	&q_i\cdot p_{ij}\cdot s_j\|\mathbf{L} \mathbf{H} \mathbf{\Theta}\|_{F}^{2}\big)
 \end{align}
where $q_i$ is the probability of node $i$ being contained in the first layer via neighborhood sampling and $p_{ij}$ is the importance sampling coefficient related to node $i$ and $j$. Moreover, $s_j$ is the probability of node $i$ being in the cache set $\mathcal{C}$.

Under Assumption 1 and 2 in \cite{zou2019layer} such that $\left\|\mathbf{H}_{i, *} \mathbf{\Theta}\right\|_{2} \leq \gamma \text { for all } i \in[|\mathcal{V}|]$ and $\left\|\mathbf{L}_{i, *}\right\|_{0} \leq \frac{C}{|\mathcal{V}|} \sum_{i=1}^{|\mathcal{V}|}\left\|\mathbf{L}_{i, *}\right\|_{0}$ and the definition of the importance sampling coefficient, we arrive 
\begin{align}
	&\mathbb{E}\left[\|\tilde{\mathbf{Z}}-\mathbf{P} \mathbf{Z}\|_{F}^{2}\right]\notag\\
	 \leq  & \sum_{i=1}^{|\mathcal{V}|}q_i\left\|\mathbf{L}_{i, *}\right\|_{0} \sum_{j=1}^{|\mathcal{V}|}p_{ij}\cdot s_j\left\|\mathbf{L}_{i, j} \mathbf{H}_{j, *} \mathbf{\Theta}\right\|_{2}^{2}  \notag\\
	 \leq&  {C}  \sum_{i=1}^{|\mathcal{V}|} \sum_{j=1}^{|\mathcal{V}|} q_i\cdot p_{ij}\cdot s_j\left\|\mathbf{L}_{i, j} \mathbf{H}_{j, *} \mathbf{\Theta}\right\|_{2}^{2} \notag\\
	 \leq &\frac{CB_{\text{out}} C_d \gamma \|\mathbf{L}\|_{F}^{2}}{|\mathcal{V}|},
\end{align}
where $C_d $ denotes the average degree and $B_{\text{out}}$ is the size of nodes at the output layer.

\subsection{Summary and Discussion}
GNS shares many advantages of various sampling methods and is able to avoid their drawbacks.
Like node-wise neighbor sampling,
it samples neighbors on each node independently and, thus, can be implemented and parallelized efficiently.
Due to the cache, GNS tends to avoid the neighborhood explosion in multi-layer GNN. GNS maintains
a global and static distribution to sample the cache, which requires only one-time computation and can be easily
amortized during the training. In contrast, LADIES computes the sampling distribution for every layer
in every mini-batch, which makes the sampling procedure expensive. Even though GNS constructs
a mini-batch with more nodes than LADIES, forward and backward computation on a mini-batch is not
the major bottleneck in many GNN models for mixed CPU-GPU training.
Even though both GNS and LazyGCN deploy caching to accelerate computation in mixed CPU-GPU training,
they use cache very differently.
GNS uses cache to reduce the number of nodes in a mini-batch to reduce computation and data movement between CPU and GPUs.
It captures majority of connectivities of nodes in a graph. LazyGCN caches and reuses
the sampled graph structure and node data. This requires a large mega-batch size to achieve good accuracy,
which makes it difficult to scale to giant graphs. Because LazyGCN uses node-wise sampling
or layer-wise sampling to sample mini-batches, it suffers from the problems inherent to
these two sampling algorithms. For example, as shown in the experiment section, LazyGCN cannot
construct a mega-batch with node-wise neighbor sampling on large graphs.

\begin{table*}[t]
\begin{threeparttable}
	\caption{Dataset statistics. \label{tab: exp-dataset}}
	\small
	\begin{tabular}{lrrccrcc}
		\toprule
		Dataset & \multicolumn{1}{c}{Nodes} & 
		          \multicolumn{1}{c}{Edges} & Avg. Deg& Feature & Classes & Multiclass & Train / Val / Test\\
		\midrule
		Yelp      &   716,847 &    6,977,410 & 10 & 300 & 100 & Yes & 0.75 / 0.10 / 0.15 \\
		Amazon    & 1,598,960 &  132,169,734 & 83 & 200 & 107 & Yes & 0.85 / 0.05 / 0.10 \\
		OAG-paper &   15,257,994& 220,126,508&  14 & 768 & 146 & Yes &0.43  / 0.05 / 0.05 \\
%		OGBN-mag &736,389&42,222,014&-&349(s)&0.85 / 0.09 / 0.06\\
		OGBN-products &  2,449,029& 123,718,280& 51& 100 & 47 & No &0.10  / 0.02 / 0.88 \\
		OGBN-Papers100M& 111,059,956 & 3,231,371,744 & 30 & 128 & 172 & No &0.01 / 0.001 / 0.002\\
		\bottomrule
	\end{tabular}
\end{threeparttable}
\end{table*}

\section{ Experiments} 
\label{sec:simulation}

\begin{table*}[t]
\begin{threeparttable}
	\small
	\caption{Performance of different sampling approaches. \label{tab:simulate}}
%	\begin{tabular}{p{2.5cm}<{}p{2.5cm}<{\centering}p{2cm}<{\centering}p{2cm}<{\centering}p{2cm}<{\centering}p{2cm}<{\centering}p{2cm}<{\centering}}
	\begin{tabular}{llccccc}
		\toprule
		\multicolumn{1}{l}{Dataset(hidden layer dimension)}  & Metric & NS & LADIES (512)&  LADIES (5000) & LazyGCN & GNS\\ \midrule
		
		\multirow{2}{*}{Yelp(512)}       
		& F1-Score(\%)            &          62.54             &           59.32&     61.04            &     35.58        & 63.20   \\   
		
		& Time per epoch (s)          &      58.5                   &               62.1 &    237.9       &  1248.7           & 23.1   \\ \midrule
		
		\multirow{2}{*}{Amazon(512)}   		
		& F1-Score(\%)           &          76.69             &               76.46&     77.05       &  31.08             &76.13   \\   
		
		& Time per epoch (s)           &   89.5                    &               613.4  &   3234.2        &   3280.2          & 42.8     \\ \midrule
		
		\multirow{2}{*}{OAG-paper(256)}
		& F1-Score(\%)          &    50.23       &    43.51  &          46.72          &     N/A        &  49.23   \\   
		& Time per epoch (s)    &    3203.2      &    2108.0 &               7956.0       &            & 819.4  \\ \midrule

		\multirow{2}{*}{OGBN-products(256)}    		
		& F1-Score(\%)           &       78.44               &        70.32 &        75.36          &   69.78           &  78.01  \\   
		
		& Time per epoch (s)          &  25.6                   &       45.4 &        223.5              &   264.2         &  11.9    \\ \midrule
		
		\multirow{2}{*}{OGBN-Papers100M(256)}      		
		& F1-Score(\%)           &        63.61               &     57.94         & 59.23    & N/A      &  63.31   \\
		& Time per epoch (s)          &  462.2                     &   152.7      &  313.2    &       &  98.5    \\ 
		
		\bottomrule          
	\end{tabular}
    \footnotesize	
	The results were obtained by training a 3-layer GraphSage with hidden state dimension as 512 on Yelp and Amazon dataset, and 256 on the rest, using four methods. We update the model with a mini-batch size of 1000 and ADAM optimizer with a learning rate of 0.003 for all training methods. We use the efficient implementation of node-wise neighbor sampling, LADIES and GNS in DGL, parallelized with multiprocessing. The number of sampling workers is $4$. The column labeled ``LADIES(512)'' means sampling 512 nodes at each layer and ``LADIES(5000)'' means sampling 5000 nodes in each layer. In GNS, the size of cached was $1\%$ of $|\mathcal{V}|$. LazyGCN runs out of GPU memory on OAG-paper and OGBN-papers100M.
	%
	%\gk{Note that GNS is often about 1\% worse than NS. This is something that should be discussed in the paper and make sure that we do not claim that there is no degradation in terms of quality.}
	%\gk{The speedup for the largest dataset are disappointing. Do we know why? Is this being discussed in the paper?}
	
\end{threeparttable}
\end{table*}

\subsection{Datasets and Setup}
We evaluate the effectiveness of GNS under inductive supervise setting on the following real-world large-scale datasets: Yelp \cite{zeng2019graphsaint}, and Amazon \cite{zeng2019graphsaint}, OAG-paper~\footnote{\url{https://s3.us-west-2.amazonaws.com/dgl-data/dataset/OAG/oag_max_paper.dgl} }, OGBN-products  \cite{hu2020open}, OGBN-papers100M \cite{hu2020open}.
OAG-paper is the paper citation graph in the medical domain extracted from the OAG graph\cite{oag}.
On each of the datasets, the task is to predict labels of the nodes in the graphs.
Table~\ref{tab: exp-dataset} provides various statistics for these datasets.

% By default, we train 3-layer GCNs with hidden state dimension as 512 on Yelp and Amazon dataset, and train 3-layer GCNs with hidden state dimension as 256 on the rest, using the four methods. We update the model with a mini-batch size of 1000 and ADAM optimizer with a learning rate of 0.003. The number of sampler is $4$ for each experiment. 
For each trial, we run the algorithm with ten epochs on Yelp, Amazon OGBN-products, OGBN-Papers100M dataset and each epoch proceeds for $\frac{ \text{\# train set} }{\text{batch size}}$ iterations.  For the OAG-paper dataset, we run the algorithm with three epochs.
We compare GNS with  node-wise neighbor sampling (used by GraphSage),  LADIES and LazyGCN for training
3-layer GraphSage. The detailed settings with respect to these four methods are summarized as follows:
\begin{itemize}
\item\textbf{GNS}: GNS is implemented by DGL \cite{wang2019dgl} and we apply GNS on all layers to
sample neighbors. The sampling fan-outs of each layer are 15, 10 for the third and second layer.
We sample nodes in the first layer
(input layer) only from the cache. The size of cached set is $1\%\cdot |\mathcal{V}|$.
\item\textbf{Node-wise neighbor sampling (NS)}~\footnote{\url{https://github.com/dmlc/dgl/tree/master/examples/pytorch/graphsage} } \cite{graphsage}:  NS is implemented by DGL \cite{wang2019dgl}. The sampling fan-outs of each layer are 15, 10 and 5. 
\item\textbf{LADIES}~\footnote{\url{https://github.com/BarclayII/dgl/tree/ladies/examples/pytorch/ladies} } \cite{zou2019layer}: LADIES is implemented by DGL \cite{wang2019dgl}. We sample 512 and 5000 nodes for LADIES per layer, respectively.
\item\textbf{LazyGCN}~\footnote{\url{https://github.com/MortezaRamezani/lazygcn}}  \cite{ramezani2020gcn}: we use
the implementation provided by the authors. We set the recycle period size as $R = 2$ and the recycling growth rate as $\rho = 1.1$. The sampler of LazyGCN is set as nodewise sampling with $15$ neighborhood nodes in each layer.
\end{itemize}
GNS, NS and LADIES are parallelized with multiprocessing. For all methods, we use the batch size of 1000.
We use two metrics to evaluate the effectiveness of sampling methods: micro F1-score to measure the accuracy and the average running time per epoch to measure the training speed.

We run all experiments on an AWS EC2 g4dn.16xlarge instance with 32 CPU cores, 256GB RAM and one NVIDIA T4 GPU.

\subsection{Experiment Results}
We evaluate test F1-score and average running time per epoch via using different methods in the case of large-scale dataset.
As is shown in Table \ref{tab:simulate}, GNS can obtain the comparable accuracy score compared to NS with $2\times-4\times$ speed in training time, using a small cache. The acceleration attributes to the smaller number of nodes in a mini-batch, especially a smaller number of nodes in the input layer (Table \ref{tab:num_nodes}). This significantly reduces the time of data copy between CPU and GPUs as well as reducing the computation overhead in a mini-batch (Figure \ref{fig:breakdown2}). In addition, a large number of input nodes have been cached in GPU, which further reduces the time used in data copy between CPU to GPU. GNS scales well to giant graphs with 100 million nodes as long as CPU memory of the machine can accomodate the graph. In contrast, LADIES cannot achieve the state-of-the-art model accuracy and its training speed is actually slower than NS on many graphs. Our experiments also show that LazyGCN cannot achieve good model accuracy with a small mini-batch size, which is not friendly to giant graphs. In addition,
The LazyGCN implementation provided by the authors fails to scale to giant graphs (e.g., OAG-paper and OGBN-products) due to the out-of-memory error even with a small mega-batch. 
Our method is robust to a small mini-batch size and can easily scale to giant graphs.

\begin{figure}
\centering
\includegraphics[width=\linewidth]{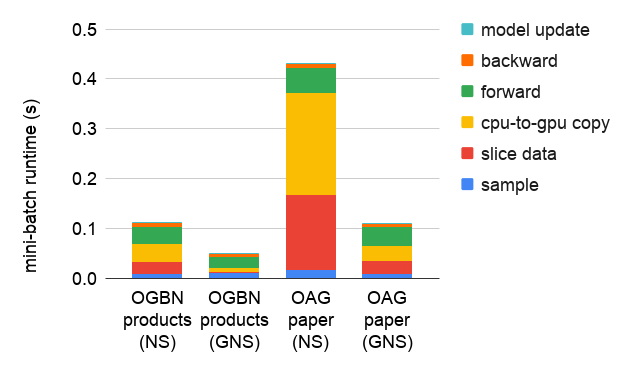}
\caption{Runtime breakdown (s) of each component in mini-batch training of NS and GNS on OGBN-products
and OAG-paper graphs.}
\label{fig:breakdown2}
\end{figure}

 \begin{table}[t]
\begin{threeparttable}
    \small
	\caption{The average number of input nodes in a mini-batch of NS and GNS as well as the average number of input nodes from the cache of GNS. \label{tab:num_nodes}}
	\begin{tabular}{lccccc}
		\toprule
		 & \#input nodes & \#input nodes & \#cached nodes \\
		 & (NS) & (GNS) & (GNS) \\
		\toprule
		Yelp & 151341 & 24150 & 5796 \\
		Amazon & 132288 & 19063 & 13986 \\
		OGBN-products       & 433928   & 88137   & 21552  \\ 
		OAG-paper   & 408854  & 102984  & 56422    \\ 
	    OGBN-Papers100M & 507405 & 155128 & 111923 \\
		\bottomrule
	\end{tabular}
\end{threeparttable}
\end{table}

We plot the convergence rate of all of the training methods on OGBN-products based on the test F1-scores (Figure~\ref{fig:f1}). In this study, LADIES samples 512 nodes per layer and GNS caches 1\% of nodes in the graph. The result indicates that GNS achieves similar convergence and accuracy as NS even with a small cache, thereby confirming the theoretical analysis in Section~\ref{analysis}, while LADIES and LazyGCN fail to converge to good model accuracy.
\begin{figure}[tb]
	\centering
	\includegraphics[width=0.8\columnwidth]{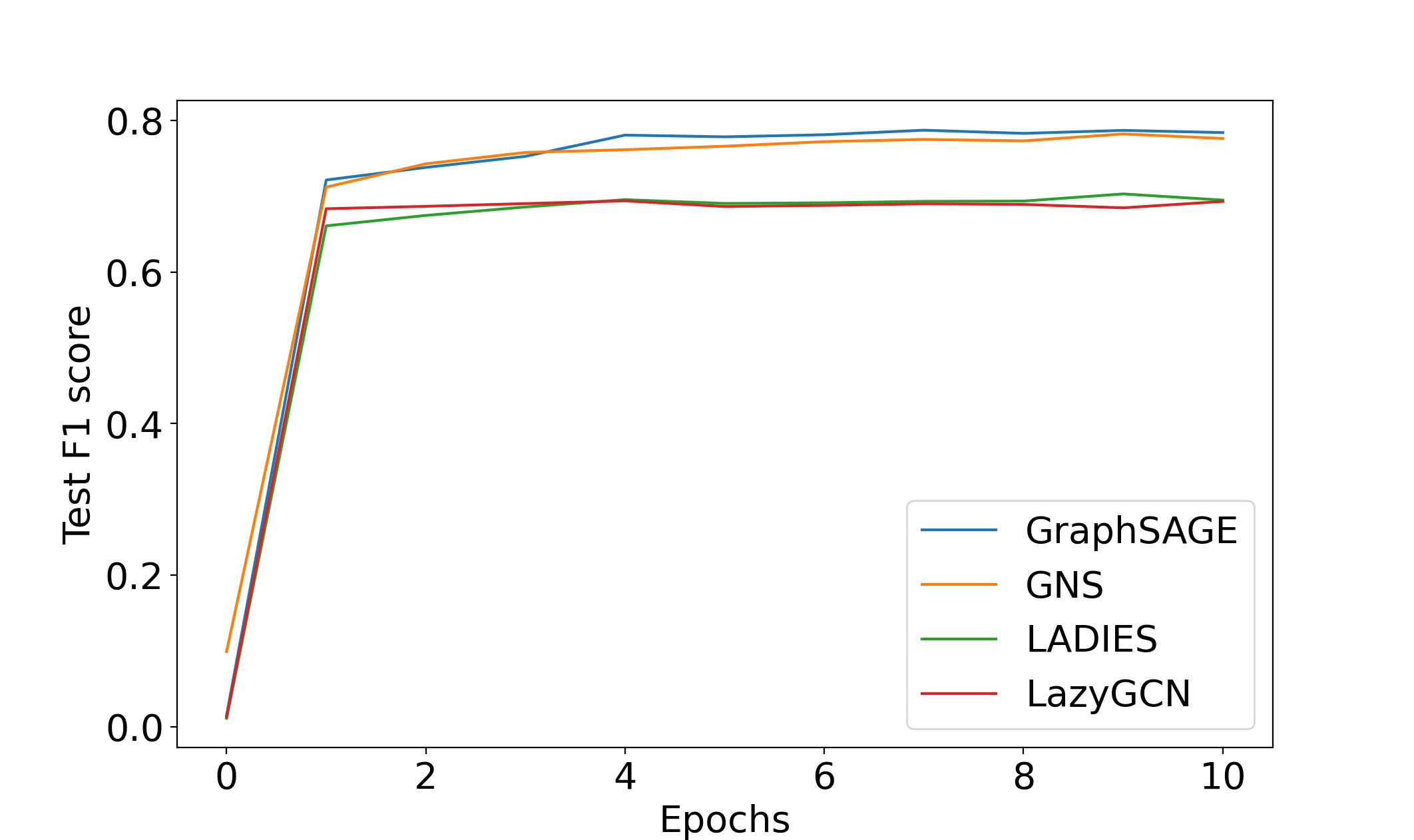}
	\caption{Comparison of the accuracy (F1 score) v.s. epochs.}
	\label{fig:f1}
\end{figure}

\begin{table}[t]
\begin{threeparttable}
    \small
	\caption{Percentage of isolated training nodes in LADIES. \label{tab:ladies}}
	%\begin {tabular}{@{\hspace{0pt}}l@{\hspace{5pt}}c@{\hspace{5pt}}c@{\hspace{5pt}}c@{\hspace{5pt}}c@{\hspace{0pt}}c@{\hspace{0pt} } }
	\begin{tabular}{lccccc}
		\toprule
		\# of sampled nodes/layer       & 256   & 512   & 1000   & 5000 &  10000  \\
		\toprule
		\% of isolated target nodes   & 52.7  & 45.2  & 24.0   & 3.9  & 0    \\ 
	
		\bottomrule
	\end{tabular}
    \footnotesize
	Percentage of isolated nodes in the first layer when training three-layer GCN on OGBN-products with LADIES.
\end{threeparttable}
\end{table}

One of reasons why LADIES suffers poor performance is that it tends to construct a mini-batch with
many isolated nodes, especially for nodes in the first layer (Table \ref{tab:ladies}).
When training three-layer GCN
on OGBN-products with LADIES, the percentage of isolated nodes in the first layer under different numbers of layerwise sampled nodes is illustrated in Tables \ref{tab:ladies}. 
\begin{figure}[tb]
	\centering
	\includegraphics[width=0.8\columnwidth]{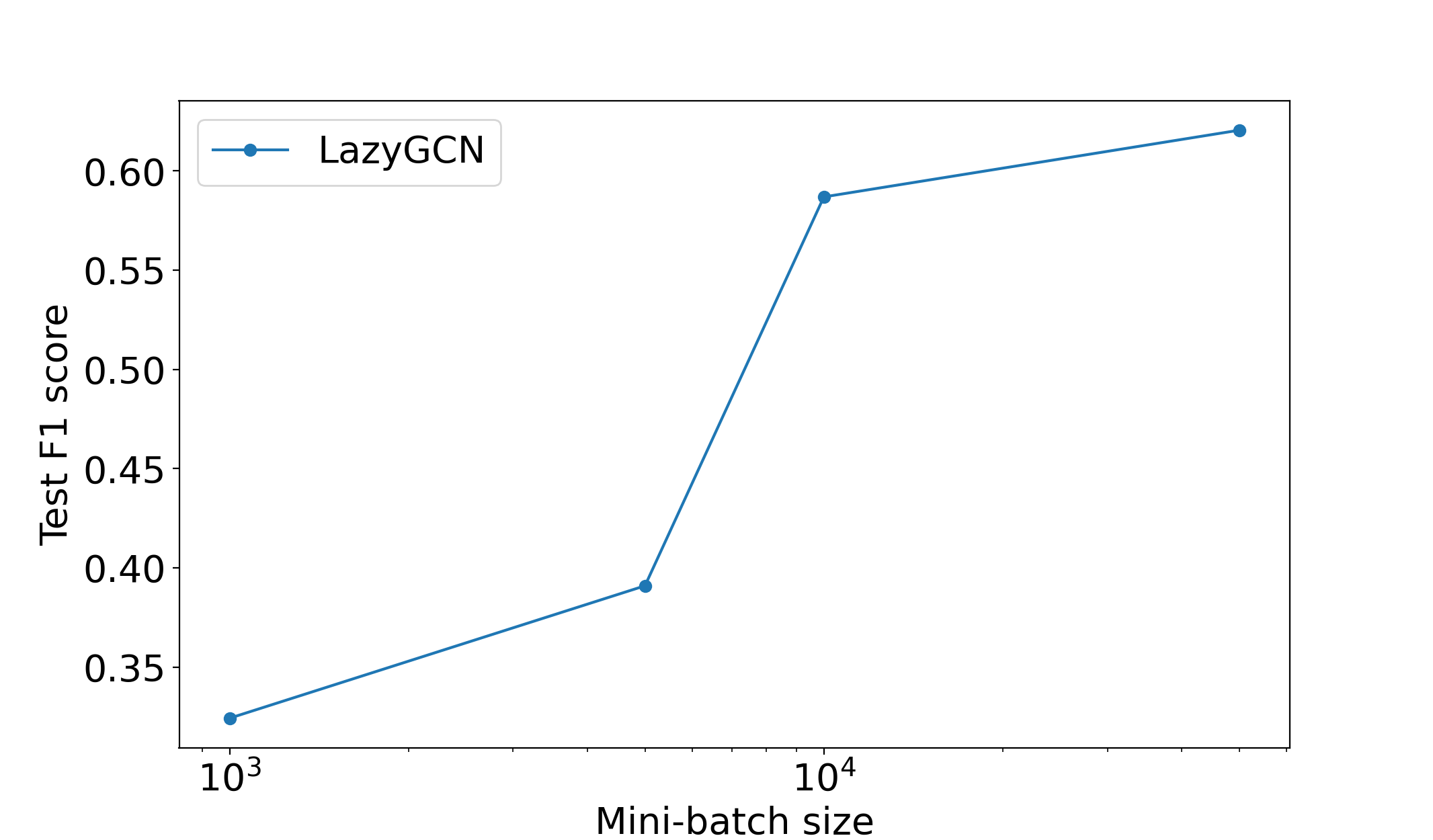}
	\caption{The effect of mini-batch size on the performance of LazyGCN on the Yelp dataset.}
	\label{fig:lazygcn}
\end{figure}

LazyGCN requires a large batch size to train GCN on large graphs, which usually leads to out of memory. Under the same setting of the paper \cite{ramezani2020gcn}, we investigate the performance of nodewise lazyGCN on Yelp dataset with different mini-batch sizes. As shown in Figure~\ref{fig:lazygcn}, LazyGCN performs poorly at the small mini-batch size. This may be caused by training with less representative graph data in mega-batch when recycling a small batch. 

\subsection{Hyperparameter study}
In this section, we explore the effect of various parameters in GNS on OGBN-products dataset. Table \ref{tab:ablation} summarizes the test F1-score with respect to different values of cache update period sizes $P$ and cache sizes. A cache size as small as $0.01\%$ can still achieve a fairly good accuracy. As long as the cache size is sufficiently large (e.g., $1\%\cdot |\mathcal{V}|$), properly reducing the frequency of updating the cache (e.g., $P=1,2,5$) does not affect performance. 
Note that it is better to get a $.1\%$ sample every epoch, than a single $1\%$ sample every 10 epochs, and $.01\%$ sample every epoch that the $0.1\%$ sample every 10 epochs.

\begin{table}[t]
\begin{threeparttable}
    \small
	\caption{GNS sensitivity to update period and cache size. \label{tab:ablation}}
	\begin{tabular}{@{\hspace{10pt}}l@{\hspace{15pt}}c@{\hspace{15pt}}c@{\hspace{15pt}}c@{\hspace{15pt}}c@{\hspace{10pt}}}
		\toprule
		& \multicolumn{4}{c}{cache update period size $P$} \\ 
		Size of cache               & $P=1$      & $P=2$      & $P=5$      & $P=10$     \\ 
		\midrule
		$|\mathcal{V}|\times 1\%$    & 78.34      & 78.40      & 78.17      & 77.54      \\ 
		$|\mathcal{V}|\times .1\%$   & 78.04     & 77.31      & 76.16      & 74.71      \\ 
		$|\mathcal{V}|\times .01\%$  & 76.29     & 72.83     & 71.60      & 71.21      \\ 
		\bottomrule
	\end{tabular}
    \footnotesize
	Performance in terms of test-set F1-score for different cache sizes and update periods. 
% 	\gk{Note that it is better to get a .1\% sample every time, than a single 1\% sample every 10 times, and .01\% sample every time that the 0.1\% sample every 10 times. This results is not surprising. It should be discussed in the paper. Does this hold for other graphs?}\jl{Is holds for other graphs}
\end{threeparttable}
\end{table}

\section{Conclusions}
In this paper, we propose a new effective sampling framework to accelerate GNN mini-batch training on giant graphs by removing the main bottleneck in mixed CPU-GPU training. GNS creates a global cache to facilitate neighbor sampling and periodically updates the cache. Therefore, it reduces data movement between CPU and GPU. We empirically demonstrate the advantages of the proposed algorithm in convergence rate, computational time, and scalability to giant graphs. Our proposed method has a significant speedup in training on large-scale dataset. We also theoretically analyze GNS and show that even with a small cache size, it enjoys a comparable convergence rate as the node-wise sampling method.

\bibliographystyle{ACM-Reference-Format}
\bibliography{Reference}  


\begin{thebibliography}{17}

%%% the Web address.

\ifx \showCODEN    \undefined \def \showCODEN     #1{\unskip}     \fi
\ifx \showDOI      \undefined \def \showDOI       #1{#1}\fi
\ifx \showISBNx    \undefined \def \showISBNx     #1{\unskip}     \fi
\ifx \showISBNxiii \undefined \def \showISBNxiii  #1{\unskip}     \fi
\ifx \showISSN     \undefined \def \showISSN      #1{\unskip}     \fi
\ifx \showLCCN     \undefined \def \showLCCN      #1{\unskip}     \fi
\ifx \shownote     \undefined \def \shownote      #1{#1}          \fi
\ifx \showarticletitle \undefined \def \showarticletitle #1{#1}   \fi
\ifx \showURL      \undefined \def \showURL       {\relax}        \fi
% invisible to TeX
\providecommand\bibfield[2]{#2}
\providecommand\bibinfo[2]{#2}
\providecommand\natexlab[1]{#1}
\providecommand\showeprint[2][]{arXiv:#2}

\bibitem[\protect\citeauthoryear{Chen, Ma, and Xiao}{Chen
  et~al\mbox{.}}{2018}]%
        {fastgcn}
\bibfield{author}{\bibinfo{person}{Jie Chen}, \bibinfo{person}{Tengfei Ma},
  {and} \bibinfo{person}{Cao Xiao}.} \bibinfo{year}{2018}\natexlab{}.
\newblock \showarticletitle{Fast{GCN}: Fast Learning with Graph Convolutional
  Networks via Importance Sampling}. In \bibinfo{booktitle}{\emph{6th
  International Conference on Learning Representations, {ICLR} 2018, Vancouver,
  BC, Canada, April 30 - May 3, 2018, Conference Track Proceedings}}.
\newblock
\urldef\tempurl%
\url{https://openreview.net/forum?id=rytstxWAW}
\showURL{%
\tempurl}


\bibitem[\protect\citeauthoryear{Devlin, Chang, Lee, and Toutanova}{Devlin
  et~al\mbox{.}}{2019}]%
        {bert}
\bibfield{author}{\bibinfo{person}{Jacob Devlin}, \bibinfo{person}{Ming-Wei
  Chang}, \bibinfo{person}{Kenton Lee}, {and} \bibinfo{person}{Kristina
  Toutanova}.} \bibinfo{year}{2019}\natexlab{}.
\newblock \showarticletitle{BERT: Pre-training of Deep Bidirectional
  Transformers for Language Understanding}.
\newblock   \bibinfo{volume}{abs/1810.04805} (\bibinfo{year}{2019}).
\newblock


\bibitem[\protect\citeauthoryear{Fey and Lenssen}{Fey and Lenssen}{2019}]%
        {pyg}
\bibfield{author}{\bibinfo{person}{Matthias Fey} {and}
  \bibinfo{person}{Jan~Eric Lenssen}.} \bibinfo{year}{2019}\natexlab{}.
\newblock \showarticletitle{Fast Graph Representation Learning with PyTorch
  Geometric}.
\newblock \bibinfo{journal}{\emph{CoRR}}  \bibinfo{volume}{abs/1903.02428}
  (\bibinfo{year}{2019}).
\newblock


\bibitem[\protect\citeauthoryear{Hamilton, Ying, and Leskovec}{Hamilton
  et~al\mbox{.}}{2017}]%
        {graphsage}
\bibfield{author}{\bibinfo{person}{William~L. Hamilton},
  \bibinfo{person}{Zhitao Ying}, {and} \bibinfo{person}{Jure Leskovec}.}
  \bibinfo{year}{2017}\natexlab{}.
\newblock \showarticletitle{Inductive Representation Learning on Large Graphs}.
  In \bibinfo{booktitle}{\emph{Advances in Neural Information Processing
  Systems 30: Annual Conference on Neural Information Processing Systems 2017,
  4-9 December 2017, Long Beach, CA, {USA}}}. \bibinfo{pages}{1025--1035}.
\newblock
\urldef\tempurl%
\url{http://papers.nips.cc/paper/6703-inductive-representation-learning-on-large-graphs}
\showURL{%
\tempurl}


\bibitem[\protect\citeauthoryear{Hochreiter and Schmidhuber}{Hochreiter and
  Schmidhuber}{1997}]%
        {hochreiter1997long}
\bibfield{author}{\bibinfo{person}{Sepp Hochreiter} {and}
  \bibinfo{person}{J{\"u}rgen Schmidhuber}.} \bibinfo{year}{1997}\natexlab{}.
\newblock \showarticletitle{Long short-term memory}.
\newblock \bibinfo{journal}{\emph{Neural computation}} \bibinfo{volume}{9},
  \bibinfo{number}{8} (\bibinfo{year}{1997}), \bibinfo{pages}{1735--1780}.
\newblock


\bibitem[\protect\citeauthoryear{Hu, Fey, Zitnik, Dong, Ren, Liu, Catasta, and
  Leskovec}{Hu et~al\mbox{.}}{2020}]%
        {hu2020open}
\bibfield{author}{\bibinfo{person}{Weihua Hu}, \bibinfo{person}{Matthias Fey},
  \bibinfo{person}{Marinka Zitnik}, \bibinfo{person}{Yuxiao Dong},
  \bibinfo{person}{Hongyu Ren}, \bibinfo{person}{Bowen Liu},
  \bibinfo{person}{Michele Catasta}, {and} \bibinfo{person}{Jure Leskovec}.}
  \bibinfo{year}{2020}\natexlab{}.
\newblock \showarticletitle{Open graph benchmark: Datasets for machine learning
  on graphs}.
\newblock \bibinfo{journal}{\emph{arXiv preprint arXiv:2005.00687}}
  (\bibinfo{year}{2020}).
\newblock


\bibitem[\protect\citeauthoryear{Kipf and Welling}{Kipf and Welling}{[n. d.]}]%
        {gcn}
\bibfield{author}{\bibinfo{person}{Thomas~N. Kipf} {and} \bibinfo{person}{Max
  Welling}.} \bibinfo{year}{[n. d.]}\natexlab{}.
\newblock \showarticletitle{Semi-Supervised Classification with Graph
  Convolutional Networks}. In \bibinfo{booktitle}{\emph{5th International
  Conference on Learning Representations (ICLR-17)}}.
\newblock


\bibitem[\protect\citeauthoryear{Kohler and Lucchi}{Kohler and Lucchi}{2017}]%
        {kohler2017sub}
\bibfield{author}{\bibinfo{person}{Jonas~Moritz Kohler} {and}
  \bibinfo{person}{Aurelien Lucchi}.} \bibinfo{year}{2017}\natexlab{}.
\newblock \showarticletitle{Sub-sampled Cubic Regularization for Non-convex
  Optimization}. In \bibinfo{booktitle}{\emph{International Conference on
  Machine Learning}}. \bibinfo{pages}{1895--1904}.
\newblock


\bibitem[\protect\citeauthoryear{Liu, Wu, Zhang, Zhou, Yang, Song, and Qi}{Liu
  et~al\mbox{.}}{2020}]%
        {liu2020bandit}
\bibfield{author}{\bibinfo{person}{Ziqi Liu}, \bibinfo{person}{Zhengwei Wu},
  \bibinfo{person}{Zhiqiang Zhang}, \bibinfo{person}{Jun Zhou},
  \bibinfo{person}{Shuang Yang}, \bibinfo{person}{Le Song}, {and}
  \bibinfo{person}{Yuan Qi}.} \bibinfo{year}{2020}\natexlab{}.
\newblock \showarticletitle{Bandit Samplers for Training Graph Neural
  Networks}.
\newblock   \bibinfo{volume}{abs/2006.05806} (\bibinfo{year}{2020}).
\newblock


\bibitem[\protect\citeauthoryear{Qi, Su, Mo, and Guibas}{Qi
  et~al\mbox{.}}{2017}]%
        {qi2017pointnet}
\bibfield{author}{\bibinfo{person}{Charles~R Qi}, \bibinfo{person}{Hao Su},
  \bibinfo{person}{Kaichun Mo}, {and} \bibinfo{person}{Leonidas~J Guibas}.}
  \bibinfo{year}{2017}\natexlab{}.
\newblock \showarticletitle{Pointnet: Deep learning on point sets for 3d
  classification and segmentation}. In \bibinfo{booktitle}{\emph{Proceedings of
  the IEEE conference on computer vision and pattern recognition}}.
  \bibinfo{pages}{652--660}.
\newblock


\bibitem[\protect\citeauthoryear{Ramezani, Cong, Mahdavi, Sivasubramaniam, and
  Kandemir}{Ramezani et~al\mbox{.}}{2020}]%
        {ramezani2020gcn}
\bibfield{author}{\bibinfo{person}{Morteza Ramezani}, \bibinfo{person}{Weilin
  Cong}, \bibinfo{person}{Mehrdad Mahdavi}, \bibinfo{person}{Anand
  Sivasubramaniam}, {and} \bibinfo{person}{Mahmut Kandemir}.}
  \bibinfo{year}{2020}\natexlab{}.
\newblock \showarticletitle{GCN meets GPU: Decoupling “When to Sample” from
  “How to Sample”}.
\newblock \bibinfo{journal}{\emph{Advances in Neural Information Processing
  Systems}}  \bibinfo{volume}{33} (\bibinfo{year}{2020}).
\newblock


\bibitem[\protect\citeauthoryear{Tang, Zhang, Yao, Li, Zhang, and Su}{Tang
  et~al\mbox{.}}{2008}]%
        {oag}
\bibfield{author}{\bibinfo{person}{Jie Tang}, \bibinfo{person}{Jing Zhang},
  \bibinfo{person}{Limin Yao}, \bibinfo{person}{Juanzi Li}, \bibinfo{person}{Li
  Zhang}, {and} \bibinfo{person}{Zhong Su}.} \bibinfo{year}{2008}\natexlab{}.
\newblock \showarticletitle{ArnetMiner: Extraction and Mining of Academic
  Social Networks}. In \bibinfo{booktitle}{\emph{Proceedings of the 14th ACM
  SIGKDD International Conference on Knowledge Discovery and Data Mining}}.
  \bibinfo{address}{New York, NY, USA}.
\newblock


\bibitem[\protect\citeauthoryear{Velickovic, Cucurull, Casanova, Romero,
  Li{\`{o}}, and Bengio}{Velickovic et~al\mbox{.}}{[n. d.]}]%
        {gat}
\bibfield{author}{\bibinfo{person}{Petar Velickovic}, \bibinfo{person}{Guillem
  Cucurull}, \bibinfo{person}{Arantxa Casanova}, \bibinfo{person}{Adriana
  Romero}, \bibinfo{person}{Pietro Li{\`{o}}}, {and} \bibinfo{person}{Yoshua
  Bengio}.} \bibinfo{year}{[n. d.]}\natexlab{}.
\newblock \showarticletitle{Graph Attention Networks}. In
  \bibinfo{booktitle}{\emph{6th International Conference on Learning
  Representations (ICLR-18)}}.
\newblock


\bibitem[\protect\citeauthoryear{Wang, Zheng, Ye, Gan, Li, Song, Zhou, Ma, Yu,
  Gai, Xiao, He, Karypis, Li, and Zhang}{Wang et~al\mbox{.}}{2019}]%
        {wang2019dgl}
\bibfield{author}{\bibinfo{person}{Minjie Wang}, \bibinfo{person}{Da Zheng},
  \bibinfo{person}{Zihao Ye}, \bibinfo{person}{Quan Gan},
  \bibinfo{person}{Mufei Li}, \bibinfo{person}{Xiang Song},
  \bibinfo{person}{Jinjing Zhou}, \bibinfo{person}{Chao Ma},
  \bibinfo{person}{Lingfan Yu}, \bibinfo{person}{Yu Gai},
  \bibinfo{person}{Tianjun Xiao}, \bibinfo{person}{Tong He},
  \bibinfo{person}{George Karypis}, \bibinfo{person}{Jinyang Li}, {and}
  \bibinfo{person}{Zheng Zhang}.} \bibinfo{year}{2019}\natexlab{}.
\newblock \showarticletitle{Deep Graph Library: A Graph-Centric,
  Highly-Performant Package for Graph Neural Networks}.
\newblock \bibinfo{journal}{\emph{arXiv preprint arXiv:1909.01315}}
  (\bibinfo{year}{2019}).
\newblock


\bibitem[\protect\citeauthoryear{Zeng, Zhou, Srivastava, Kannan, and
  Prasanna}{Zeng et~al\mbox{.}}{2019}]%
        {zeng2019graphsaint}
\bibfield{author}{\bibinfo{person}{Hanqing Zeng}, \bibinfo{person}{Hongkuan
  Zhou}, \bibinfo{person}{Ajitesh Srivastava}, \bibinfo{person}{Rajgopal
  Kannan}, {and} \bibinfo{person}{Viktor Prasanna}.}
  \bibinfo{year}{2019}\natexlab{}.
\newblock \showarticletitle{Graphsaint: Graph sampling based inductive learning
  method}.
\newblock \bibinfo{journal}{\emph{arXiv preprint arXiv:1907.04931}}
  (\bibinfo{year}{2019}).
\newblock


\bibitem[\protect\citeauthoryear{Zheng, Ma, Wang, Zhou, Su, Song, Gan, Zhang,
  and Karypis}{Zheng et~al\mbox{.}}{2020}]%
        {zheng2020distdgl}
\bibfield{author}{\bibinfo{person}{Da Zheng}, \bibinfo{person}{Chao Ma},
  \bibinfo{person}{Minjie Wang}, \bibinfo{person}{Jinjing Zhou},
  \bibinfo{person}{Qidong Su}, \bibinfo{person}{Xiang Song},
  \bibinfo{person}{Quan Gan}, \bibinfo{person}{Zheng Zhang}, {and}
  \bibinfo{person}{George Karypis}.} \bibinfo{year}{2020}\natexlab{}.
\newblock \showarticletitle{DistDGL: Distributed Graph Neural Network Training
  for Billion-Scale Graphs}.
\newblock \bibinfo{journal}{\emph{arXiv preprint arXiv:2010.05337}}
  (\bibinfo{year}{2020}).
\newblock


\bibitem[\protect\citeauthoryear{Zou, Hu, Wang, Jiang, Sun, and Gu}{Zou
  et~al\mbox{.}}{2019}]%
        {zou2019layer}
\bibfield{author}{\bibinfo{person}{Difan Zou}, \bibinfo{person}{Ziniu Hu},
  \bibinfo{person}{Yewen Wang}, \bibinfo{person}{Song Jiang},
  \bibinfo{person}{Yizhou Sun}, {and} \bibinfo{person}{Quanquan Gu}.}
  \bibinfo{year}{2019}\natexlab{}.
\newblock \showarticletitle{Layer-dependent importance sampling for training
  deep and large graph convolutional networks}. In
  \bibinfo{booktitle}{\emph{Advances in Neural Information Processing
  Systems}}. \bibinfo{pages}{11249--11259}.
\newblock


\end{thebibliography}
\newpage

\appendix

\numberwithin{equation}{section}

\section{Proof of Theorem 1}\label{app_proof}
By extending Lemma 4 in \cite{ramezani2020gcn}, we conclude that $\mathbb{E}[\|\nabla J(\hat{\boldsymbol{\theta}})\|^{2}]$ depends on $\frac{1}{T}\sum_{t =1}^{T}\E[\normn{\nabla \widetilde{J}_{\mathcal{B}}(\boldsymbol{\theta}_t) - \nabla J(\boldsymbol{\theta}_t)}{}^2]$.
To complete the proof of theorem \ref{theorem1}, we first present few vital lemmas. To be specific, Lemma \ref{1} characterizes the bounds on $\frac{1}{T}\sum_{t =1}^{T}$ $\E[\normn{\nabla \widetilde{J}_{\mathcal{B}}(\boldsymbol{\theta}_t) - \nabla J(\boldsymbol{\theta}_t)}{}^2]$.
\begin{lemma}\label{1}
Denote  $N^i_{\ell}$ as the number of the neighborhood nodes with respect to $i\in\mathcal{V}$ sampling uniformly at random at $\ell$-th layer.  The cached nodes in the set $\mathcal{C}$ with the size of $ |\mathcal{C}|$ are sampled without replacement according to $p^{\text{cache}}_{v}$. The dimension of node feature is denoted as $n$ and the size of  min-batch is denoted as $B$. Define  $\widetilde{C} = |\mathcal{C}|/|\mathcal{V}|$ and  $C_d = \sum_{v_i\in\mathcal{V} } \mathrm{deg}(v_i)/|\mathcal{V}|$ with the constant $c>0$. Under Assumption \ref{ass:lp}, the expected mean-square error of stochastic gradient $\nabla \widetilde{J}_{\mathcal{B}}(\boldsymbol{\theta})$ derived from Algorithm \ref{alg:full} to the full gradient  is bounded by
\begin{align}\label{mse}
&\mathrm{MSE} := \frac{1}{T}\sum_{t =1}^{T}\E\left[\norm{\nabla \widetilde{J}_{\mathcal{B}}(\boldsymbol{\theta}_t) - \nabla J(\boldsymbol{\theta}_t)}{}^2\right]\notag\\
\leq&\mathcal{O}\left( L^{'2}_{f}  \frac{\log (4 n / \delta)+1 / 2}{B}\right)+\mathcal{O}\left( L_f^ {'2} L_g^ {4}  \frac{\log (4 n / \delta)+1 / 2}{c\widetilde{C} C_dN_1^jN_2^i} \right)\notag\\
&+\mathcal{O}\left(   L_g^ {'2}  L_f^ {2}   \frac{\log (4 n / \delta) }{c\widetilde{C} C_dN_1^jN_2^i} \right).
\end{align}

\end{lemma}

\begin{proof}
According to the inequality, $\|\boldsymbol{a}+\boldsymbol{b}\|^{2} \leq 2\|\boldsymbol{a}\|^{2}+2\|\boldsymbol{b}\|^{2}$, $\mathrm{MSE}$ (\ref{mse}) can be decomposed into two parts:
\begin{align}\label{mse1}
&\mathrm{MSE} := \frac{1}{T}\sum_{t =1}^{T}\E\left[\norm{\nabla \widetilde{J}_{\mathcal{B}}(\boldsymbol{\theta}_t) - \nabla J(\boldsymbol{\theta}_t)}{}^2\right]\notag\\
\leq&  \frac{2}{T} \sum_{t=1}^{T}\mathbb{E}\left[\left\|\nabla \widetilde{J}_{\mathcal{B}}\left(\boldsymbol{\theta}_{t}\right)-\nabla J_{\mathcal{B}}\left(\boldsymbol{\theta}_{t}\right)\right\|^{2}\right]\notag\\
&+\frac{2}{T} \sum_{t=1}^{T} {\mathbb{E}\left[\left\|\nabla J_{\mathcal{B}}\left(\boldsymbol{\theta}_{t}\right)-\nabla J\left(\boldsymbol{\theta}_{t}\right)\right\|^{2}\right]}
\end{align}
Two terms in (\ref{mse1}) are bounded by Lemma \ref{2} and \ref{3}, respectively.
\end{proof}

\begin{lemma}\label{2}
 Based on the notations in Lemma \ref{1}, with probability exceeding $1-\delta$ we have
\begin{align}
&\E[\normn{\nabla \widetilde{J}_{\mathcal{B}}(\boldsymbol{\theta}) - \nabla J_{\mathcal{B}}(\boldsymbol{\theta})}{}^2]\notag\\
\leq& 128  L_f^ {'2} L_g^ {4}   \frac{\log (4 n / \delta)+1 / 2}{c\widetilde{C} C_dN_1^jN_2^i}+ 64  L_g^ {'2}  L_f^ {2}\frac{\log (4 n / \delta) }{c\widetilde{C} C_dN_1^jN_2^i}.
\end{align}
\end{lemma}
\begin{proof}
The proof of Lemma \ref{2} is provided in Appendix \ref{proof l2}.
\end{proof}

\begin{lemma}\label{3}
	Based on the notations in Lemma \ref{1}, with probability exceeding $1-\delta$ we have
	\begin{align}
	\E[\normn{\nabla  {J}_{\mathcal{B}}(\boldsymbol{\theta}) - \nabla J(\boldsymbol{\theta})}{}^2]
	\leq 128    L_f^ {'2}  \frac{\log (4 n / \delta)+1 / 2}{B}.
	\end{align}
\end{lemma}
\begin{proof}
Besides the definition of $\nabla J_{\mathcal{B}}(\boldsymbol{\theta})$ (\ref{batch}), $\nabla J(\boldsymbol{\theta})$ is given as
\begin{align}
\nabla J(\boldsymbol{\theta})=\frac{1}{|\mathcal{V}|} \sum_{i \in \mathcal{V}} A_{1}^iA_{2}^i,
\end{align}
where $ A_{1}^i, A_{2}^i$ are represented by (\ref{a1}) and (\ref{a2}), respectively.

For simplicity,we denote
\[\mathbb{E}_{\mathcal{V}_{\mathcal{B}} \sim \mathcal{V}}\big[\mathbb{E}_{j \sim \mathcal{N}(i), \forall i \in \mathcal{V}_{\mathcal{B}}}[\mathbb{E}_{k \sim \mathcal{N}
	(j), \forall j\in \mathcal{N}(i)} [\cdot]]\big] \]
 as $\E[\cdot]$.
\begin{align}
&\E[\normn{\nabla  {J}_{\mathcal{B}}(\boldsymbol{\theta}) - \nabla J(\boldsymbol{\theta})}{}^2]\\
\leq &\E\left[\norm{\frac{1}{B} \sum_{i \in \mathcal{V}_{\mathcal{B}}} A_{1}^iA_{2}^i - \frac{1}{|\mathcal{V}|} \sum_{i \in \mathcal{V}} A_{1}^iA_{2}^i }{}^2\right],\\
\leq& 128   L_f^ {'2}  \frac{\log (4 n / \delta)+1 / 2}{c\widetilde{C} C_dN_1^jN_2^i},
\end{align}
where the last inequality comes from Lemma \ref{lemma:pre}.
\end{proof}

\section{Proof of Lemma \ref{2} }\label{proof l2}
To bound $\E[\normn{\nabla \widetilde{J}_{\mathcal{B}}(\boldsymbol{\theta}) - \nabla J_{\mathcal{B}}(\boldsymbol{\theta})}{}^2]$, we begin with the definition of $\nabla \widetilde{J}_{\mathcal{B}}(\boldsymbol{\theta})$ and $\nabla J_{\mathcal{B}}(\boldsymbol{\theta}){}$, which is given by
\begin{align}\label{batch}
\nabla J_{\mathcal{B}}(\boldsymbol{\theta})=\frac{1}{B} \sum_{i \in \mathcal{V}_{\mathcal{B}}} A_{1}^iA_{2}^i,
\end{align}
where
\begin{align}
A_{1}^i& = \nabla f_{i}\left(\frac{1}{|\mathcal{N}(i)|} \sum_{j \in \mathcal{N}(i)} \frac{1}{|\mathcal{N}(j)|}\sum_{k \in \mathcal{N}(j)}g_{jk}(\boldsymbol{\theta})\right), \label{a1}\\
A_{2}^i& = \frac{1}{|\mathcal{N}(i)|} \sum_{j \in \mathcal{N}(i)} \frac{1}{|\mathcal{N}(j)|}\sum_{k \in \mathcal{N}(j)}\nabla g_{jk}(\boldsymbol{\theta}).\label{a2}
\end{align}
\begin{align}
\nabla \widetilde{J}_{\mathcal{B}}(\boldsymbol{\theta}) = \frac{1}{B} \sum_{i \in \mathcal{V}_{\mathcal{B}}}  B_{1}^iB_{2}^i
\end{align}
where
\begin{align}
B_{1}^i& = \nabla f_i\left( \frac{1}{|\mathcal{N}_2^{\mathcal{C}}(i)|} \sum_{j \in \mathcal{N}_2^{\mathcal{C}}(i)} \frac{1}{|\mathcal{N}_1
	^{\mathcal{C}}(j)|}\sum_{k \in \mathcal{N}_1
	^{\mathcal{C}}(j)}p_k^{(1)} g_{jk}(\boldsymbol{\theta})\right), \\
B_{2}^i& = \frac{1}{|\mathcal{N}_2^{\mathcal{C}}(i)|} \sum_{j \in \mathcal{N}_2^{\mathcal{C}}(i)} \frac{1}{|\mathcal{N}_1
	^{\mathcal{C}}(j)|}\sum_{k \in \mathcal{N}_1  
	^{\mathcal{C}}(j)}p_k^{(1)}\nabla  g_{jk}(\boldsymbol{\theta}).
\end{align}
For simplicity,we denote
\[\mathbb{E}_{\mathcal{V}_{\mathcal{B}} \sim \mathcal{V}}\left[\mathbb{E}_{j \sim \mathcal{N}(i), \forall i \in \mathcal{V}_{\mathcal{B}}}[\mathbb{E}_{k \sim \mathcal{N}
	(j), \forall j\in \mathcal{N}(i)}[\cdot]]\right]\]
as $\E[\cdot]$.

Based on the inequalities $\left\|\frac{1}{n} \sum_{i=1}^{n} \boldsymbol{a}_{i}\right\| \leq \frac{1}{n} \sum_{i=1}^{n}\left\|\boldsymbol{a}_{i}\right\|$, $\|\boldsymbol{a}+\boldsymbol{b}\|^{2} \leq$
$2\|\boldsymbol{a}\|+2\|\boldsymbol{b}\|,$  and $\| \boldsymbol{ab}\| \leq\|\boldsymbol{a}\|\|\boldsymbol{b}\| $, we arrive
\begin{align}
&\E[\normn{\nabla \widetilde{J}_{\mathcal{B}}(\boldsymbol{\theta}) - \nabla J_{\mathcal{B}}(\boldsymbol{\theta})}{}^2]\notag\\
=&2 \mathbb{E}\left[\left\|B_{1}^{i}\right\|^{2}\right] \mathbb{E}\left[\left\|B_{2}^{i}-A_{2}^{i}\right\|^{2}\right]\notag\\
&+2 \mathbb{E}\left[\left\|B_{1}^{i}-A_{1}^{i}\right\|^{2}\right] \mathbb{E}\left[\left\|A_{2}^{i}\right\|^{2}\right].\label{ej}
\end{align}
We shall bound two terms in (\ref{ej}).
\begin{enumerate}
\item The first term: for $\mathbb{E}\left[\left\|B_{1}^{i}\right\|^{2}\right]$, we have 
\begin{align}
\mathbb{E}\left[\left\|B_{1}^i\right\|^{2}\right]\leq L_{f}^{2}.\label{t0}
\end{align}
 In terms of $\mathbb{E}\left[\left\|B_{2}^{i}-A_{2}^{i}\right\|^{2}\right]$, we have
\begin{align}
&\mathbb{E}\left[\left\|B_{2}^{i}-A_{2}^{i}\right\|^{2}\right]\notag\\
 = & L_f^ {'2} \mathbb{E}\Bigg[ \bigg\| \sum_{j \in \mathcal{N}_2^{\mathcal{C}}(i)} \sum_{k \in \mathcal{N}_1
 	^{\mathcal{C}}(j)} \frac{p_k^{(1)}}{N_{\mathcal{C}_2}^{i} N_{\mathcal{C}_1}^{j} } \nabla g_{jk}(\boldsymbol{\theta}) \notag\\ -
 &\sum_{j \in \mathcal{N}(i)}\sum_{k \in \mathcal{N}(j)} \frac{1}{(N^i)^2} \nabla g_{jk}(\boldsymbol{\theta})
 \bigg\|^2\Bigg]\notag\\
 \leq& 64  L_g^ {'2}  \frac{\log (4 n / \delta) }{c\widetilde{C} C_dN_1^jN_2^i}\label{t4},
\end{align}
where the last inequality comes from Lemma \ref{lemma:pre}.
\item  The second term: for $\mathbb{E}\left[\left\|A_{2}^{i}\right\|^{2}\right]$, we have 
\begin{align}
\mathbb{E}\left[\left\|A_{2}^i\right\|^{2}\right]\leq L_{g}^{2}.\label{t3}
\end{align}
In terms of $\mathbb{E}\left[\left\|B_{1}^{i}-A_{1}^{i}\right\|^{2}\right]$, we have
\begin{align}
&\mathbb{E}\left[\left\|B_{1}^{i}-A_{1}^{i}\right\|^{2}\right]\notag\\
= & L_f^ {'2} \mathbb{E}\Bigg[ \bigg\| \sum_{j \in \mathcal{N}_2^{\mathcal{C}}(i)} \sum_{k \in \mathcal{N}_1
	^{\mathcal{C}}(j)} \frac{p_k^{(1)}}{N_{\mathcal{C}_2}^{i}   N_{\mathcal{C}_1}^{j} }  g_{jk}(\boldsymbol{\theta}) \notag\\ -
&\sum_{j \in \mathcal{N}(i)}\sum_{k \in \mathcal{N}(j)} \frac{1}{(N^i)^2}  g_{jk}(\boldsymbol{\theta})
\bigg\|^2\Bigg]\notag\\
\leq& 128  L_g^ {2}  L_f^ {'2}  \frac{\log (4 n / \delta)+1 / 2}{c\widetilde{C} C_dN_1^jN_2^i}\label{t1},
\end{align}
where the last inequality comes from Lemma \ref{lemma:pre}.
\end{enumerate}

By integrating inequalities (\ref{t0}), (\ref{t1}), (\ref{t3}), (\ref{t4}), it yields

\begin{align}
&\E[\normn{\nabla \widetilde{J}_{\mathcal{B}}(\boldsymbol{\theta}) - \nabla J_{\mathcal{B}}(\boldsymbol{\theta})}{}^2]\notag\\
 \leq& 128  L_f^ {'2} L_g^ {4}   \frac{\log (4 n / \delta)+1 / 2}{c\widetilde{C} C_dN_1^jN_2^i}+ 64  L_g^ {'2}  L_f^ {2}\frac{\log (4 n / \delta) }{c\widetilde{C} C_dN_1^jN_2^i}.
\end{align}
%\begin{lemma}
%Given the sampled set $\mathcal{S}_1,\mathcal{S}_2$ and the cached set $\mathcal{C}$ where the element $k\in\mathcal{C}$  is sampled corresponding to $p_k$. Denote the pro Define the sub-sampled function as 
%\begin{align}
%Q_{\mathcal{S}}(\theta) = \frac{1}{\mathcal{S}_2}\sum_{i\in\mathcal{S}_1}\frac{1}{|\mathcal{S}_1\cap\mathcal{C}|}
%\end{align}
%\end{lemma}

\begin{lemma}\label{lemma:pre}
Based on the notations in Lemma \ref{1}, with probability $1-\delta$ we have
 \begin{align}
 &\mathbb{E}\Bigg[ \bigg\| \sum_{j \in \mathcal{N}_2^{\mathcal{C}}(i)} \sum_{k \in \mathcal{N}_1
 	^{\mathcal{C}}(j)} \frac{p_k^{(1)}}{N_{\mathcal{C}_2}^{i}   N_{\mathcal{C}_1}^{j} }  g_{jk}(\boldsymbol{\theta}) \notag\\
 	&-\E_{j \sim \mathcal{N}(i),k \sim \mathcal{N}(j)}  [g_{jk}(\boldsymbol{\theta})]
 \bigg\|^2\Bigg]\notag\\
 \leq & 8 \sqrt{2} L_{g} \sqrt{\frac{\log (4 n / \delta)+1 / 2}{c\widetilde{C} C_dN_1^jN_2^i}}     \quad \forall ~i\in\mathcal{V}\\
 &\mathbb{E}\Bigg[ \bigg\| \sum_{j \in \mathcal{N}_2^{\mathcal{C}}(i)} \sum_{k \in \mathcal{N}_1
 	^{\mathcal{C}}(j)} \frac{p_k^{(1)}}{N_{\mathcal{C}_2}^{i}   N_{\mathcal{C}_1}^{j} }  \nabla g_{jk}(\boldsymbol{\theta})  \notag\\
 	&-\E_{j \sim \mathcal{N}(i),k \sim \mathcal{N}(j)}  [\nabla g_{jk}(\boldsymbol{\theta})]
 \bigg\|^2\Bigg]\notag\\
 \leq & 8   L'_{g} \sqrt{\frac{\log (4 n / \delta) }{c\widetilde{C} C_dN_1^jN_2^i}}      \quad \forall  ~i\in\mathcal{V},\label{eq2}
 \end{align}
 where $C_d = \sum_{v_i\in\mathcal{V} } \mathrm{deg}(v_i)/|\mathcal{V}|$ with the constant $c>0$.
\end{lemma}
\begin{proof}
	The proof is derived from the proof of Lemma 6 in \cite{kohler2017sub} and Lemma 10 and 11 in \cite{ramezani2020gcn}. Based on the definition of $p_k^{(1)}$, for $i\in\mathcal{V}$, we have
	\begin{align}
	&\E\left[   \sum_{j \in \mathcal{N}_2^{\mathcal{C}}(i)} \sum_{k \in \mathcal{N}_1
		^{\mathcal{C}}(j)} \frac{p_k^{(1)}}{N_{\mathcal{C}_2}^{i}   N_{\mathcal{C}_1}^{j} }  g_{jk}(\boldsymbol{\theta})   \right]\notag\\ =& \E_{j \sim \mathcal{N}(i),k \sim \mathcal{N}(j)}  [g_{jk}(\boldsymbol{\theta})].
	\end{align}
Similarly, there is 
	\begin{align}
	&\E\left[   \sum_{j \in \mathcal{N}_2^{\mathcal{C}}(i)} \sum_{k \in \mathcal{N}_1
		^{\mathcal{C}}(j)} \frac{p_k^{(1)}}{N_{\mathcal{C}_2}^{i}   N_{\mathcal{C}_1}^{j} }  \nabla g_{jk}(\boldsymbol{\theta})   \right]\notag\\ =& \E_{j \sim \mathcal{N}(i),k \sim \mathcal{N}(j)}  [\nabla  g_{jk}(\boldsymbol{\theta})].
	\end{align}
	The value of $ N_{\mathcal{C}_1}^{j} $ and $ N_{\mathcal{C}_2}^{j} $ depend on the predefined number of neighborhood nodes sampled uniformly at random in each layer, i.e., $N_2^i, N_1^j$, the ratio of cached nodes to the whole graph nodes, i.e.,  $\widetilde{C} = |\mathcal{C}|/|\mathcal{V}|$ and the sampling probability $\mathcal{P}$. Thus, $ N_{\mathcal{C}_1}^{j} $ can be approximated by $N_{\mathcal{C}_1}^{j} =c\widetilde{C} C_dN_1^j$
	where $C_d = \sum_{v_i\in\mathcal{V} } \mathrm{deg}(v_i)/|\mathcal{V}|$ with the constant $c>0$. Based on the above, the proof can be completed by extending the proof of Lemma 10 and 11 in \cite{ramezani2020gcn}.
	
Given the sampled set $\mathcal{S}_1,\mathcal{S}_2$ and 
$$
V_{S}(\boldsymbol{\theta}) = \frac{1}{|S_1|\cdot|S_2|} \sum_{i\in \mathcal{S}_1}\sum_{j\in \mathcal{S}_2} V_{ij}(\boldsymbol{\theta}),
$$ where $ V_{ij}(\boldsymbol{\theta})\in\mathbb{R}^n$ is $L_v$-Lipschitz continuous for all $i,j$,
 the proof can be completed via Bernstein's bound with a sub-Gaussian tail, given by
	we have
\begin{align}
&\mathbb{P}\left(\left\|V_{S}(\boldsymbol{\theta})-\E [V_{S}(\boldsymbol{\theta})]\right\| \geq \epsilon\right)\notag\\
 \leq &4n \cdot \exp \left(-\frac{\epsilon^{2} |S_1|\cdot|S_2|}{64 L_v^{2}}+\frac{1}{2}\right),
\end{align}
where $\epsilon\leq 2L_v$.

	Finally, let $\delta$ as the upper bound of Bernstein inequality
	\begin{align}
	\delta=4 n \cdot \exp \left(-\frac{\epsilon^{2} |S_1|\cdot|S_2|}{64 L_v^{2}}+\frac{1}{2}\right).
	\end{align}
	Therefore, we have
	\begin{align}
	\epsilon=8 \sqrt{2} L_v \sqrt{\frac{\log (4n / \delta)+1 / 2}{|S_1|\cdot|S_2|}}.
	\end{align}
	The inequality (\ref{eq2}) can be clarified in similar way.
\end{proof}

\end{document}